\documentclass{article}

\PassOptionsToPackage{numbers, compress}{natbib}
\usepackage[preprint]{neurips_2022}

\usepackage{algorithm}
\usepackage{algorithmic}
\usepackage[utf8]{inputenc} % allow utf-8 input
\usepackage[T1]{fontenc}    % use 8-bit T1 fonts
\usepackage{hyperref}       % hyperlinks
\usepackage{url}            % simple URL typesetting
\usepackage{booktabs}       % professional-quality tables
\usepackage{amsfonts}       % blackboard math symbols
\usepackage{nicefrac}       % compact symbols for 1/2, etc.
\usepackage{microtype}      % microtypography
\usepackage{xcolor}         % colors
\usepackage{authblk}
\usepackage[pdftex]{graphicx}
\newtheorem{theorem}{Theorem}
\newtheorem{lemma}{Lemma}
\newtheorem{definition}{Definition}

\newtheorem{assumption}{Assumption}
\newtheorem{remark}{Remark}
\newtheorem{claim}{Claim}
\usepackage{enumitem}
\newenvironment{proof}{{\noindent\it Proof}.}{\hfill $\square$\par} 
\newenvironment{proof(sketch)}{{\noindent\it Proof(sketch)}.}{\hfill $\square$\par}

\usepackage{amsmath}
\usepackage{mathtools}
\usepackage{caption}
\usepackage{subcaption}
\usepackage{colortbl}
\usepackage{multirow}

\title{Coresets for Relational Data and The Applications}

\author[1]{\textbf{Jiaxiang Chen}}

\author[2]{\textbf{Qingyuan Yang}}

\author[1]{\textbf{Ruomin Huang}}

\author[2]{\textbf{Hu Ding\thanks{Corresponding author.}}\ \ }

\affil[1]{School of Data Science}
\affil[2]{School of Computer Science and Technology}
\affil[ ]{University of Science and Technology of China}
\affil[ ]{\texttt{\{czar, yangqingyuan, hrm\}}\texttt{@mail.ustc.edu.cn},  \texttt{\href{mailto:huding@ustc.edu.cn}{huding@ustc.edu.cn}}}

\begin{document}

\maketitle

\begin{abstract}
 A coreset is a small set that can approximately preserve the structure of the original input data set. Therefore we can run our algorithm on a coreset so as to reduce the total computational complexity. Conventional coreset techniques assume that the input data set is available to process explicitly.  
 However, this assumption may not hold in real-world scenarios. In this paper, we consider the problem of coresets construction over relational data. Namely, the data is decoupled into several relational tables, and it could be very expensive to directly materialize the data matrix by joining the tables. We propose a novel approach called ``aggregation tree with pseudo-cube'' that can build a coreset from bottom to up. Moreover, our approach can neatly circumvent several troublesome issues of relational learning problems [Khamis et al., PODS 2019]. Under some mild assumptions, we show that our coreset approach can be applied for the machine learning tasks, such as clustering, logistic regression and SVM. 
\end{abstract}

\section{Introduction}
\label{sec-intro}
%1.Large-scale data
As the  rapid development of information technologies, we are often confronted with large-scale data in many practical  scenarios. To reduce the computational complexity, a number of  data summarization techniques have been proposed~\cite{DBLP:journals/corr/Phillips16}. \textbf{Coreset} is a well-known sampling technique for compressing large-scale data sets~\cite{DBLP:journals/ki/MunteanuS18, DBLP:journals/widm/Feldman20}. Roughly speaking, for a given large data set $P$, the coreset approach is to construct a new (weighted) set $\tilde{P}$ that  approximately preserves the structure of $P$, and meanwhile the size  $|\tilde{P}|$ is much smaller than $ |P|$. Therefore, if we run an existing algorithm on  $\tilde{P}$ rather than $P$, the runtime can be significantly reduced and the computational quality (e.g., the optimization objectives for the machine learning tasks~\cite{sra2012optimization}) can be approximately guaranteed. In the past decades, the coreset techniques have been successfully applied for solving many optimization problems, such as $k$-means clustering~\cite{DBLP:journals/widm/Feldman20, DBLP:conf/stoc/Cohen-AddadSS21}, logistic regression~\cite{DBLP:conf/nips/HugginsCB16, DBLP:conf/gi/MunteanuSSW19}, Gaussian mixture models~\cite{DBLP:journals/jmlr/LucicFKF17}, and continual learning~\cite{borsos2020coresets}.

%2.Relational database
Usually we assume that the input data is stored in a matrix such that the coreset algorithm can easily access the whole data. But this assumption may not hold in real-world scenarios. For example, according to the Kaggle 2017 ``State of Data Science Survey'', \textbf{$\textbf{65.5\%}$ of the data sets are relational data}. Namely, the data is decoupled into several relational tables, and we cannot obtain the data matrix (which is also called the ``\textbf{design matrix}'') unless joining all the tables. Relational database has a very large and fast growing market in this big data era. It is expected to reach USD $122.38$ billion by 2027~\cite{market}. Relational database has several  favorable properties~\cite{sumathi2007fundamentals}, e.g., the decoupled relational data can save a large amount of space, and it is friendly to several common statistical queries such as the  aggregate functions $\mathtt{COUNT}, \mathtt{SUM}$ and $\mathtt{AVG}$ in SQL. In recent years, the study on ``in-database learning'' has gained a significant amount of attention in the areas of machine learning and data management~\cite{DBLP:conf/amw/0001NOS17,DBLP:conf/pods/Khamis0NOS18,DBLP:conf/aistats/CurtinM0NOS20,DBLP:conf/apocs/KhamisIMPS21a}.

\textbf{However, constructing a coreset for relational database is much more challenging. }
Suppose there are $s>1$ relational tables and the size of the largest table is $N$. Let $n$ be the number of the tuples obtained by joining all the tables, and then $n$ can be as large as $O(N^s)$~\cite{DBLP:journals/siamcomp/AtseriasGM13} (we illustrate an instance of joining two tables in Table~\ref{tab}). Obviously it will take extremely large time and space complexities to join all the tables and train a model on such scale of data.
To remedy this issue, a straightforward idea is trying to complete the computing task without explicitly constructing the design matrix. Unfortunately, even for some very simple computational tasks, 
their implementations on relational data can be NP-hard. 

To see why these implementations are so hard, we can consider the following simple clustering assignment task (which is often used as a  building block for coresets construction~\cite{DBLP:journals/widm/Feldman20, DBLP:conf/stoc/Cohen-AddadSS21}): suppose the $k$ cluster centers have already been given, and we want to determine the size of each cluster. Note that if we have the design matrix, this task is trivial (just assign each point to its nearest center, and then calculate the cluster sizes based on the assignment). However, if the data is decoupled into several relational tables, the problem of computing the assignment can be quite complicated. Khamis et al.~\cite{DBLP:journals/tods/KhamisCMNNOS20} recently defined the problem of answering \textbf{F}unctional \textbf{A}ggregate \textbf{Q}ueries (FAQ) in which some of the input factors are defined by a collection of \textbf{A}dditive \textbf{I}nequalities\footnote{Each additive inequality defines a constraint for the query; for example, in SVM, the ``positive'' class is defined by the inequality $\langle x, \omega\rangle >0$, where $\omega$ is the normal vector of the separating hyperplane.} between variables (denote by $\mathrm{FAQ}\mbox{-}\mathrm{AI}(m)$, where $m$ is the number of additive inequalities). 
The  problem of computing the clustering assignment can be formulated as an instance of $\mathrm{FAQ}\mbox{-}\mathrm{AI}(k-1)$ (to determine the cluster membership for one point (tuple), we need to certify that its distance to its own cluster center is lower than the distances to the other $k-1$ centers). 
Recently, Khamis et al.~\cite{abo2021approximate} showed that evaluating even a $\mathrm{FAQ}\mbox{-}\mathrm{AI}(1)$ instance is \#P-hard,  and approximating a $\mathrm{FAQ}\mbox{-}\mathrm{AI}(m)$ instance within any finite factor is NP-hard for any $m>1$. Moseley et al.~\cite{moseley2021relational} also showed that even approximating the clustering assignment to any factor is NP-hard for $k\geq 3$.  An intuitive understanding is that the query satisfying those inequalities require to access the information of each tuple on all the dimensions, where it is almost equivalent to constructing the entire design matrix. 

\subsection{Our Contributions}
\label{sec-our}

We consider {\em Empirical Risk Minimization (ERM)} problems in machine learning~\cite{DBLP:conf/nips/Vapnik91}. Let $\mathbb{R}^d$   be the data space. %Suppose we have the 
The training set $P=\{p_1, p_2, \cdots, p_n\}\subset \mathbb{R}^d$. But we assume that this set $P$ is not explicitly given, where it is decoupled into $s$ relational tables (the formal definitions are shown in Section~\ref{sec-pre}). 
The goal is to learn the hypothesis $\theta$ (from the hypothesis space $\mathbb{H}$) so as to minimize the {\em empirical risk} 
\begin{eqnarray}
F(\theta)=\frac{1}{n}\sum^n_{i=1}f(\theta, p_i), \label{for-er}
\end{eqnarray}
where $f(\cdot,  \cdot)$ is the non-negative real-valued \emph{loss function}. 

Several coresets techniques on relational data have been studied very recently. 
Samadian et al.~\cite{samadian2020unconditional} showed that the simple uniform sampling yields a coreset for regularized loss minimization problems (note the uniform sampling can be efficiently implemented for relational database~\cite{zhao2018random}).  Their sample size is  $\Theta\left(n^{\kappa} \cdot \mathtt{dim}\right)$, where $\kappa\in (0,1)$ (usually $\kappa$ is set to be $1/2$ in practice) and $\mathtt{dim}$ is the VC dimension of loss function (usually it is $\Theta(d)$). Thus the size can be $\Theta\left(\sqrt{n} \cdot d\right)$, which is too large especially for high dimensional data. 
Curtin et al.~\cite{DBLP:conf/aistats/CurtinM0NOS20} constructed a coreset for $k$-means clustering by building a weighted grid among the input $s$ tables; the coreset  yields a $9$-approximation for the clustering objective. Independently, Ding et al.~\cite{DBLP:conf/icml/DingLHL16} also applied the ``grid'' idea to achieve a $(9+\epsilon)$-approximation for $k$-means clustering with distributed dimensions (attributes). The major drawback of this ``grid'' idea is that the resulting coreset size can be as large as $k^s$ which is exponential in the number of tables $s$. Moseley et al.~\cite{moseley2021relational} recently proposed a coreset for $k$-means clustering on relational data by using the $k$-means++ initialization method~\cite{Arthur2007kmeansTA}, however, their approximation ratio is too high ($>400$) for real-world applications. Moreover, the coresets techniques proposed in~\cite{DBLP:conf/aistats/CurtinM0NOS20,DBLP:conf/icml/DingLHL16,moseley2021relational} can only handle $k$-means clustering, and it is unclear that whether they can be extended for other machine learning problems as the form of (\ref{for-er}). 

In this paper, we aim to develop an efficient and general coreset construction method for optimizing (\ref{for-er}) on relational data. First, we observe that real-world data sets often have low intrinsic dimensions (e.g, nearby a low-dimensional manifold)~\cite{belkin2003problems}. 
Our coreset approach is inspired by the  well-known Gonzalez's algorithm for $k$-center clustering~\cite{DBLP:journals/tcs/Gonzalez85}. The algorithm greedily selects $k$ points from the input point set, and the $k$ balls centered at these selected $k$ points with some appropriate radius can cover the whole input; if the input data has a low intrinsic dimension (e.g., the doubling dimension in Definition~\ref{def-DD}), the radius is no larger than an upper bound depending on $k$. Therefore, the set of cluster centers (each center has the weight equal to the corresponding cluster size) can serve as a coreset for the ERM problem (\ref{for-er}), where the error yielded from the coreset is determined by the radius. Actually this $k$-center clustering based intuition has been used to construct the  coresets for various applications before~\cite{DBLP:conf/pods/IndykMMM14,sener2018active,DBLP:conf/iclr/ColemanYMMBLLZ20}. 

\textbf{However, we have to address two challenging issues for realizing this approach on relational data. } 
First, the greedy selection procedure for the $k$ cluster centers cannot be directly implemented for relational data. Second, it is hard to obtain the size of each cluster especially when the balls have overlap (the $k$ balls can partition the space into as many as $2^k$ parts). Actually, both of these two issues can be regarded as the instances of $\mathrm{FAQ}\mbox{-}\mathrm{AI}(k-1)$ which are NP-hard to compute~\cite{abo2021approximate,moseley2021relational}. 

Our approach relies on a novel and easy-to-implement  structure called ``\textbf{aggregation tree with pseudo-cube}''. We build the tree from bottom to up, where each relational table represents a leaf. Informally speaking, each node consists of a set of $k$ cluster centers which is obtained by merging its two children; each center also associates with   a ``pseudo-cube'' region in the space. Eventually, the root of the tree yields a qualified coreset for the  implicit design matrix. The aggregation manner can help us to avoid  building the high-complexity  grid coreset as~\cite{DBLP:conf/aistats/CurtinM0NOS20,DBLP:conf/icml/DingLHL16}. Also, with the aid of ``pseudo-cube'', we can efficiently estimate the size of each cluster without tackling the troublesome  $\mathrm{FAQ}\mbox{-}\mathrm{AI}$ counting issue~\cite{abo2021approximate,moseley2021relational}.   

Comparing with the previous coresets methods, our approach enjoys several significant advantages. For example, our approach can deal with more general applications. In fact, for most ERM problems in the form of  (\ref{for-er}) under some mild assumptions, we can construct their coresets by applying our approach. Also our coresets have much lower complexities. It is worth to emphasize that our coreset size is independent of the dimensionality $d$; instead, it only depends on the doubling dimension of the input data.

\subsection{Other Related Works}

The earliest research on relational data was proposed by Codd~\cite{DBLP:journals/cacm/Codd83}. A fundamental topic that is closely related to machine learning on relational data is how to take uniform and independent samples from the full join results~\cite{DBLP:conf/sigmod/ChaudhuriMN99}.
Recently, Zhao et al.~\cite{zhao2018random} proposed a  random walk approach for this problem with acyclic join; Chen and Yi~\cite{chen2020random} generalized the result to some specific cyclic join scenarios.

Beside the aforementioned works~\cite{DBLP:conf/amw/0001NOS17,DBLP:conf/pods/Khamis0NOS18,DBLP:conf/aistats/CurtinM0NOS20,DBLP:conf/icml/DingLHL16,DBLP:journals/tods/KhamisCMNNOS20,abo2021approximate,moseley2021relational}, there also exist a number of results on other machine learning problems over relational data.  Khamis et al.~\cite{abo2021relational} and Yang et al.~\cite{yang2020towards} respectively considered training SVMs and SVMs with Gaussian kernel on relational data. 
For the problem of linear regression on relational data, it is common to use the factorization techniques~\cite{schleich2016learning,abo2018database}. 
The algorithms for training Gaussian Mixture Models and Neural Networks on relational data were also studied recently~\cite{cheng2019nonlinear,DBLP:conf/icde/ChengKZ021}. We also refer the reader to the survey paper~\cite{DBLP:journals/corr/abs-1911-06577} for more details on learning over relational data. 

\section{Preliminaries}
\label{sec-pre}
Suppose the training set for the problem (\ref{for-er}) is a set $P$ of $n$ points in $\mathbb{R}^d$, and it is decoupled into $s$ relational tables $\{T_{1}, \dots, T_{s}\}$ with the feature (attribute) sets $\{D_{1}, \dots, D_{s}\}$. Let  $D = \bigcup_{i}D_{i}$ and therefore the size $|D|=d$. Actually, each table $T_i$ can be viewed as projection of $P$ onto a subspace  spanned by $D_{i}$. To materialize the design matrix of $P$, a straightforward way is to compute the join over these $s$ tables. With a slight abuse of notations, we still use ``$P$'' to denote the design matrix. We also let $[s]=\{1,2,\ldots, s\}$ for simplicity.

\begin{definition}[Join] 
The join over the given $s$ tables returns a $n\times d$ design matrix $P = T_{1} \bowtie \cdots \bowtie T_{s}$, such that for any vector (point) $p\in\mathbb{R}^d$, $p\in P$ if and only if  $\forall i \in[s]$,  $\operatorname{Proj}_{D_{i}}(p) \in T_{i}$, where $\operatorname{Proj}_{D_{i}}(p)$ is the projection of $p$ on the subspace spanned by the features of $D_i$.
\end{definition}

Table~\ref{tab} is a simple illustration for joining two relational tables. 
To have a more intuitive understanding of join, we can generate a \textbf{hypergraph $\mathcal{G} = (\mathcal{V}, \mathcal{E})$}. Each vertex of $\mathcal{V}$  corresponds to an individual feature of $D$; each hyperedge of $\mathcal{E}$ corresponds to  an individual relational table $T_i$, and it  connects all the vertices (features) of  $D_i$. Then we can define the \textbf{acyclic} and \textbf{cyclic} join queries.  
A join is acyclic if the  hypergraph $\mathcal{G} = (\mathcal{V}, \mathcal{E})$ can be ablated to be empty by performing the 
 following operation iteratively: if 
 there exists a vertex that connects with only one hyperedge, remove this vertex together with the hyperedge. Otherwise, the join is cyclic. 
A cyclic query usually is extremely difficult and has a  much higher complexity than that of an acyclic query. For example, for a cyclic join, it is even NP-hard to determine that whether the join is empty or not~\cite{DBLP:journals/jacm/Marx13}. Fortunately, most real-world joins are acyclic, which allow us to take full advantage of relational data.  Similar with most of the previous articles on relational  learning~\cite{abo2021approximate,abo2021relational}, we also assume that the join of the given tables is acyclic in this paper.

\begin{table}
    \begin{subtable}[t]{0.3\linewidth}
        \flushright
        \begin{tabular}[t]{|c|c|}
            \hline
            \rowcolor{yellow} \multicolumn{2}{|c|}{$T_1$} \\ 
            \hline
            \rowcolor{yellow} $d_1$ & $d_2$\\
            \hline
            1 & 1\\
            \hline
            2 & 1\\
            \hline
            2 & 2\\
            \hline
            3 & 3\\
            \hline
        \end{tabular}
    \end{subtable}
    \hspace{\fill}
    \begin{subtable}[t]{0.3\linewidth}
        \centering
        \begin{tabular}[t]{|c|c|}
            \hline
            \rowcolor{yellow} \multicolumn{2}{|c|}{$T_2$} \\ 
            \hline
            \rowcolor{yellow} $d_2$ & $d_3$\\
            \hline
            1 & 1\\
            \hline
            1 & 4\\
            \hline
            3 & 1\\
            \hline
            3 & 3\\
            \hline
        \end{tabular}
    \end{subtable}
    \hspace{\fill}
    \begin{subtable}[t]{0.3\linewidth}
        \flushleft
        \begin{tabular}[t]{|c|c|c|}
            \hline
            \rowcolor{yellow} \multicolumn{3}{|c|}{$T_1\bowtie T_2$} \\ 
            \hline
            \rowcolor{yellow} $d_1$ & $d_2$ & $d_3$\\
            \hline
            1 & 1 & 1\\
            \hline
            1 & 1 & 4\\
            \hline
            2 & 1 & 1\\
            \hline
            2 & 1 & 4\\
            \hline
            3 & 3 & 1\\
            \hline
            3 & 3 & 3\\
            \hline
        \end{tabular}
    \end{subtable}
    \caption{An illustration for the join over two tables.}
    \label{tab}
    \vspace{-0.2in}
\end{table}

\textbf{Counting} is one of the most common aggregation queries on relational data. It returns  the number of tuples that satisfy some particular conditions. 
The counting on an acyclic join (without additive inequalities) can be implemented effectively via dynamic programming~\cite{DBLP:conf/stoc/Dyer03}.  
But when the additive inequalities are included, the counting problem can be much more challenging and we may even need to materialize the whole design matrix~\cite{abo2021approximate}.  For example, to count the tuples of $\{t\in T_1\bowtie T_2 \And \|t\|_{2}^2\le 5\}$ in Table~\ref{tab}, we need to check the whole design matrix $T_1\bowtie T_2$ for selecting the tuples that satisfy the constraint ``$\|t\|_{2}^2\le 5$''.

For the machine learning problems studied in this paper, we have the following assumption for the loss function in (\ref{for-er}).

\begin{assumption}[Continuity]
\label{assump-lip-continue}
There exist real constants $\alpha, z \ge 0, \beta\in [0,1)$, such that for any $p, q\in \mathbb{R}^d$ and any $\theta$ in the hypothesis space, we have
	\begin{equation}\label{eq:Lips-continue}
	|f(\theta,p) - f(\theta,q)| \le \alpha\|p- q\|^{z} + \beta|f(\theta, q)|, 
	\end{equation}
	where  $\|\cdot\|$ is the Euclidean norm in the space.
\end{assumption}
\begin{remark}
 Different machine learning problems have different values for $\alpha, \beta$ and $z$.  For example, for the $k$-means clustering, we have $z=2, \beta=\epsilon$ and $\alpha=O(\frac{1}{\epsilon})$ (where $\epsilon$ can be any small number in $(0,1)$); for the $k$-median clustering,  we have $z=1, \beta=0$ and $\alpha=1$. Actually for a large number of problems, $\beta$ is small or even $0$, e.g., the  logistic regression and SVM with soft margin problems have the value $\beta=0$. 
\end{remark}

We define the following coreset with both multiplicative error and additive error. We are aware that the standard coresets usually only have multiplicative errors~\cite{DBLP:journals/ki/MunteanuS18, DBLP:journals/widm/Feldman20}. However, the deviation bounds for the ERM problems with finite training data set only yield additive error guarantees and the additive error is usually acceptable in practice~\cite{DBLP:conf/icml/BachemLH017,DBLP:conf/nips/TelgarskyD13}. So the coresets with additive error have been also proposed recently~\cite{DBLP:conf/kdd/BachemL018}. Let $P=\{p_1, p_2,\dots, p_n\}$ be the training set (which is not explicitly given), and denote by $\Delta$ the diameter of $P$ (i.e., the largest pairwise distance of $P$). 

\begin{definition}[$\left(\epsilon_1, \epsilon_2\right)_z$-Coreset]
\label{def-coreset}
 Suppose $\epsilon_1\in (0,1)$ and $\epsilon_2, z > 0$. The $\left(\epsilon_1, \epsilon_2\right)_z$-coreset, denoted as $\tilde{P}$, is a  point set $\{c_1, \cdots, c_{|\tilde{P}|}\}$  with a weight vector $W=[w_1, w_2, \dots, w_{|\tilde{P}|}]$ satisfying that
\begin{eqnarray}
\tilde{F}(\theta)\in (1\pm\epsilon_1)F(\theta) \pm \epsilon_2 \Delta^{z},  
\end{eqnarray}
for any $\theta$ in the hypothesis space $\mathbb{H}$, where $\tilde{F}(\theta)=\frac{1}{\sum^{|\tilde{P}|}_{i=1}w_i}\sum^{|\tilde{P}|}_{i=1}w_i f(\theta, c_i)$.
\end{definition}

Usually we want the coreset size $|\tilde{P}|$ to be much smaller than $|P|$. So we can run our algorithm on $\tilde{P}$ and save a large amount of running time. 

As mentioned in Section~\ref{sec-our}, we also assume that the training set $P$ has a low intrinsic dimension. ``Doubling dimension'' is a widely used measure for indicating the intrinsic dimension of a given data set~\cite{luukkainen1998assouad}. Roughly speaking, it measures the growth rate of the given data in the space. Recently it has gained a lot of attention for studying its relation to coresets and machine learning problems~\cite{li2001improved,DBLP:conf/focs/HuangJLW18,DBLP:journals/jacm/Cohen-AddadFS21}. For any $c\in\mathbb{R}^d$ and $r\geq 0$, we use $\mathbb{B}(c, r)$ to denote the ball centered at $c$ with radius $r$.

\begin{definition}[Doubling Dimension]
\label{def-DD}
The doubling dimension of a data set $P$ is the smallest number $\rho>0$, such that for any $p \in P$ and $r \geq 0, P \cap \mathbb{B}(p, 2 r)$ is always covered by the union of at most $2^{\rho}$ balls with radius $r$.
\end{definition}

It is easy to obtain the following proposition by recursively applying Definition~\ref{def-DD} $\log \frac{\Delta}{r}$ times.

 \begin{claim}\label{claim-doubling}
For a given data set $P$ and radius $r>0$, if $P$ has the doubling dimension $\rho$, then it can be covered by $(\frac{\Delta}{r})^{\rho}$ balls of radius $r$. 
\end{claim}

\section{Relational Coreset via Aggregation Tree}\label{sec-aggtree}

In this section, we present an efficient coreset construction method for  relational data. 
First, we introduce the technique of aggregation tree with pseudo-cube.  Actually there is an obstacle for obtaining the coreset from the aggregation tree. The weight of each point of the coreset is determined by the size of its corresponding pseudo-cube; however, the pseudo-cubes may have overlap and it is challenging to separate them in the environment of relational data. 
To explain our idea more clearly, we temporally ignore this ``overlap'' issue and show the ``ideal'' coreset construction result in Section~\ref{construction}. Then we show that the overlap issue can be efficiently solved by using random sampling and some novel geometric  insights in Section~\ref{weighting}.

Recall that our input is a set of   $s$ relational tables $\{T_{1}, \dots, T_{s}\}$ with the feature (attribute) sets $\{D_{1}, \dots, D_{s}\}$. Also $D = \bigcup_{i}D_{i}$ and $|D|=d$. For ease of presentation, we generate a new feature set  $\hat{D}_i$ for each $D_i$ as follows: initially, $\hat{D}_1=D_1$; starting from $i=2$ to $s$, we let $\hat{D}_i=D_i\setminus (\cup^{i-1}_{l=1}\hat{D}_l)$. It is easy to see that $\bigcup_{i}\hat{D}_{i}=D$, and $\forall i, j \in [s], i\ne j$, $\hat{D}_{i} \cap \hat{D}_{j} = \emptyset$. With a slight abuse of notations, we also use $\hat{D}_i$ to represent  the  subspace spanned by the features of $\hat{D}_i$ (so the Cartesian product $\hat{D}_1\times\cdots\times \hat{D}_s$ is the the whole space $\mathbb{R}^d$).
For any point set $Q\subset\mathbb{R}^d$ (resp., any point $c\in \mathbb{R}^d$) and any subspace $H$ of $\mathbb{R}^d$, we use $\operatorname{Proj}_{H}(Q)$ (resp., $\operatorname{Proj}_{H}(c)$) to denote the projection of $Q$ (resp., $c$) onto $H$.

\subsection{Coreset Construction}\label{construction}
Since our coreset construction algorithm is closely related to the Gonzalez's algorithm for $k$-center clustering~\cite{DBLP:journals/tcs/Gonzalez85}, we briefly introduce it first. 
Initially, it arbitrarily selects a point from $P$, say $c_1$, and sets $C=\{c_1\}$; then it iteratively selects a new point that is farthest to the set $C$, i.e., $\arg\max_{p\in P}\min_{q\in C}||p-q||$, and adds it to $C$. After $k$ iterations, we obtain $k$ points in $C$ (suppose $C=\{c_1, \cdots, c_k\}$). The Gonzalez's algorithm yields a $2$-approximation for $k$-center clustering. If the optimal radius of $k$-center clustering on $P$ is $r_{\mathtt{opt}}$, then $P$ can be covered by the $k$ balls $\mathbb{B}(c_1, 2r_{\mathtt{opt}}), \cdots, \mathbb{B}(c_k, 2r_{\mathtt{opt}})$. Together with  Claim~\ref{claim-doubling}, for any given radius $r>0$,
we know that if we run the Gonzalez's algorithm on $P$ with setting $k=(\frac{2\Delta}{r})^{\rho}$, the set $P$ can be covered by the balls $\mathbb{B}(c_1, r), \cdots, \mathbb{B}(c_k, r)$ (since $r_{opt}\leq r/2$). If we set $r$ to be small enough, the obtained $C$ can be a good approximation (informally a coreset) for $P$. 

However, as mentioned in Section~\ref{sec-intro}, such $k$-center clustering approach cannot be implemented on relational data since it is equivalent to an instance of  $\mathrm{FAQ}\mbox{-}\mathrm{AI}(k-1)$. A straightfoward idea to remedy this issue is to run the Gonzalez's algorithm in each subspace $\hat{D}_{i}$, and then compute the Cartesian product to build a grid of size $k^s$ as the methods of~\cite{DBLP:conf/aistats/CurtinM0NOS20,DBLP:conf/icml/DingLHL16}. But this approach will result in an exponentially large coreset (e.g., there are $s=8$ tables and $k=1000$). Below, we introduce our algorithm that can achieve a coreset with quasi-polynomial size. 

\begin{figure}[ht]
    \centering
    \includegraphics[width=0.65\linewidth]{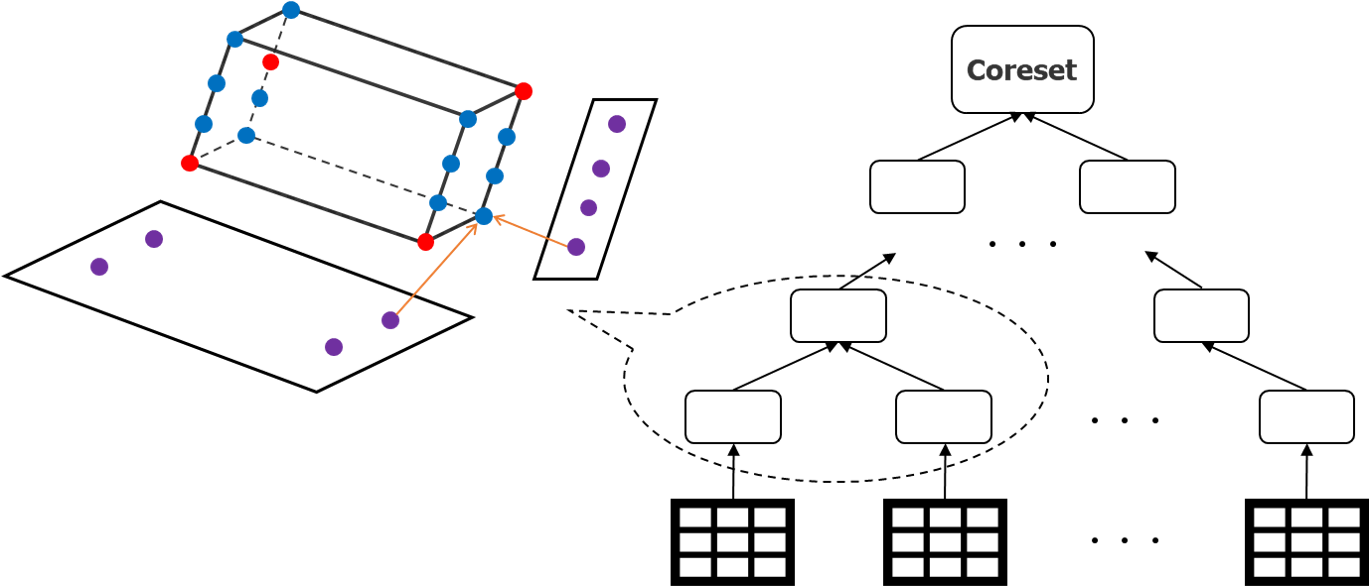}
      \caption{For the example in the figure, $k=4$. When merging two children, we have $C_{\nu_l}$ and $C_{\nu_r}$ from two disjoint subspaces (the purple points). Then we build the $4\times 4$ grid, and select $C_{\nu_p}$ (the red points) by running the Gonzalez's algorithm on the grid points.}
  \label{fig-tree}
  %\vspace{-0.1in}
\end{figure}

\paragraph{Our high-level idea.} We construct a \textbf{bottom-to-top tree $\mathcal{T}$} where each relational table $T_i$ corresponds to a leaf.  For each node $\nu\in\mathcal{T}$, it is associated with a set of components:  an index set $I_\nu\subseteq [s]$, the spanned subspace $H_\nu=\Pi_{i\in I_\nu}\hat{D}_i$, and a collection of centers $C_\nu=\{c_{\nu, 1}, \cdots, c_{\nu, k}\}\subset H_\nu$. The center set $C_\nu$ determines a set of pseudo-cubes (as Definition~\ref{def-PC} below) in $H_\nu$. Obviously, for a leaf node (table) $\nu$ of $T_i$, the index set $I_\nu=\{i\}$ and the space $H_\nu=\hat{D}_i$; we directly run the Gonzalez's algorithm on the projection of $P$ on $\hat{D}_i$ to obtain the set $C_\nu$. Then, we grow the tree $\mathcal{T}$ from these leaves. For each parent node $\nu_p$, we compute a grid $C_{\nu_l} \times C_{\nu_r}$ of size $k^2$ by taking the Cartesian product of the center sets from its two children $\nu_l$ and $\nu_r$; then we remove the ``empty'' grid points (we will explain the meaning of ``empty'' after Definition~\ref{def-PC}) and run the Gonzalez's algorithm on the remaining grid points, which are denoted as $\widetilde{C_{\nu_l} \times C_{\nu_r}}$, to achieve the set $C_{\nu_p}$. The index set $I_{\nu_p}=I_{\nu_l}\cup I_{\nu_r}$, and the subspace $H_{\nu_p}=H_{\nu_l}\times H_{\nu_r}$. Finally, the set $C_{\nu_0}$ of the root node $\nu_0$ yields a coreset for the point set $P$. 
See an illustration of the aggregation tree in Figure~\ref{fig-tree}. 

It is worth noting that our aggregation tree approach is fundamentally different to the hierarchical decomposition tree methods that were commonly studied in doubling metrics~\cite{har2006fast,chan2016hierarchical} and other tree structures like $kd$-tree~\cite{DBLP:books/lib/BergCKO08}. They build their trees from top to bottom, while our approach is from bottom to top. Actually, our approach can be viewed as a ``compressed'' version of the grid coreset methods~\cite{DBLP:conf/aistats/CurtinM0NOS20,DBLP:conf/icml/DingLHL16}, where we deal with a grid of size at most  $k^2$ each time. Also, the accumulated error can be also bounded since the height of $\mathcal{T}$ is $O(\log s)$.  

\begin{definition}[Pseudo-cube]
\label{def-PC}
Given an index set $I\subseteq [s]$, a point $c\in  H=\prod_{i\in I} \hat{D}_{i}$, and a number $r\geq 0$, we define the pseudo-cube 
\begin{eqnarray}
     \mathtt{PC}_I(c,r) = \prod_{i\in I} \mathbb{B}(\operatorname{Proj}_{\hat{D}_i}(c), r), 
\end{eqnarray}
where $\operatorname{Proj}_{\hat{D}_i}(c)$ is the projection of $c$ onto the subspace $\hat{D}_i$. 
\end{definition}
A nice property of pseudo-cube is that   it is a Cartesian product of a set of regions from those subspaces $\hat{D}_i$s. More importantly, it does not involve any cross-table constraints or additive inequalities.  
Consequently, counting the size of a pseudo-cube is easy to implement over the relational data. Another property is that the Cartesian product of two pseudo-cubes is still a pseudo-cube, but in a new subspace with higher dimension. When we take the grid of the two center sets from two children nodes, we also obtain $k^2$ pseudo-cubes. For each pseudo-cube, we count its size (i.e., the size of its intersection with $P$); if the size is $0$, we call the corresponding grid point is ``empty'' and remove it before running the Gonzalez's algorithm on the grid. 

Let $\mathtt{d}(A, B): = \max_{a\in A}\min_{b \in B}|| a - b ||$ for any two sets $A$ and $B$ in $\mathbb{R}^d$ (i.e., the directed Hausdorff Distance). 
First, for each $h=0, 1, \cdots, \left \lceil \log s \right \rceil$, we define a key value
 \begin{eqnarray}
L_{h}: &=& \max \{\mathtt{d}(\operatorname{Proj}_{H_\nu} (P), C_\nu)\mid \nu \text{ is at the $h$-th level of $\mathcal{T}$}\}.\label{for-l}
\end{eqnarray}
By using these $L_h$s, we present our coreset construction approach in Algorithm~\ref{alg: aggtree}.

\begin{algorithm}[h t b]
	\caption{\textsc{Aggregation Tree}}
	\label{alg: aggtree}
	\begin{algorithmic}
	\STATE {\bfseries Input:} A set of relational tables $\{T_1, \ldots, T_s\}$ and a parameter $k$ (the coreset size).
    \STATE {\bfseries Output:} A weighted point set as the coreset.
    \begin{enumerate}[fullwidth]
    
    \item Initialize an empty aggregation tree $\mathcal{T}$; when a node $\nu$ of $\mathcal{T}$ is constructed, it is associated with an index set $I_{\nu}$ and its corresponding subspace $H_{\nu}$, and a set of $k$ centers $C_{\nu} = \{c_{\nu,1},\ldots, c_{\nu,k}\}\subset H_{\nu}$. 
        
    \item Construct $s$ leaf nodes $\{\nu_1, \ldots, \nu_s\}$ corresponding to the $s$ tables at the $0$-th level: run the Gonzalez's algorithm to select $k$ centers for each table $T_i$ on the subspace $\hat{D}_i$. Also obtain the value $L_0$ defined in (\ref{for-l}).
        
    \item Let $\mathtt{Children} = \{\nu_1, \ldots, \nu_s\}$ and $\mathtt{Parent} = \emptyset$. Initialize $h=0$ to indicate the current level on $\mathcal{T}$.
        
    \item For $h=1$ to $\left \lceil \log s \right \rceil$:
    \begin{enumerate}[fullwidth, itemindent=1em]
        \item  Initialize $l_h=0$;
    
        \item While $|\mathtt{Children}|>=2$:\label{loop}
        
        \begin{enumerate}[fullwidth, itemindent=2em]
            \item Select two different nodes $\nu_i, \nu_j$ from $\mathtt{Children}$, and $\mathtt{Children} = \mathtt{Children}\setminus \{\nu_i, \nu_j\}$;
            
            \item Compute the Cartesian product to construct a grid $C_{\nu_i} \times C_{\nu_j}$ with $k^2$ centers in the space $H_{\nu_i} \times H_{\nu_j}$;
            
            \item For each $c\in C_{\nu_i} \times C_{\nu_j}$, perform the operation ``$\mathtt{COUNT}$'' to count the size of
            $\Big\{p\in P\mid \operatorname{Proj}_{H_{\nu_i}}(p) \in \mathtt{PC}_{I_{\nu_i}}(\operatorname{Proj}_{H_{\nu_i}}(c ), L_{h  - 1}) \And \operatorname{Proj}_{H_{\nu_j}}(p) \in \mathtt{PC}_{I_{\nu_j}}(\operatorname{Proj}_{H_{\nu_j}}(c ), L_{h  - 1})\Big\}$;
            
            \item Remove $c $ from $  C_{\nu_i} \times C_{\nu_j}$ if the size is $0$, and then obtain  $\widetilde{C_{\nu_i} \times C_{\nu_j}}$ that is the  set of remaining non-empty grid points;
            
            \item Construct a parent node $\nu$: run the Gonzalez's algorithm on $\widetilde{C_{\nu_i} \times C_{\nu_j}}$ to obtain a set $C_v$ of $k$ centers, let $l_h = \max(l_h, \mathtt{d}(\widetilde{C_{\nu_i} \times C_{\nu_j}},C_{\nu}))$, $I_{\nu} = I_{\nu_i}\cup I_{\nu_j}$ and $ H_{\nu} = H_{\nu_i}\times H_{\nu_j}$;
            
            \item $\mathtt{Parent}=\mathtt{Parent}\cup \{\nu\}$;
            
        \end{enumerate}
        
        \item $L_h = \sqrt{2^{h}}(l_{h} + \sqrt{2} L_{h-1})$;
        
        \item $\mathtt{Children} = \mathtt{Children} \cup \mathtt{Parent}, \mathtt{Parent}= \emptyset$;
        
    \end{enumerate}
    
        \item Let $\nu_0$ be the root node; for each $c_{\nu_0, i}\in C_{\nu_0}$, we assign a weight $w_i=|P\cap \mathtt{PC}_{I_{\nu_0}}(c_{\nu_0, i}, L_{\left \lceil \log s \right \rceil}) |$ (actually, this is not an accurate expression since the pseudo-cubes may have overlap with each other, and we will address this issue in Section~\ref{weighting}).\label{step}
    \end{enumerate}
    
    \RETURN the set $C_{\nu_0}$ with the weights $\{w_1, \cdots, w_k\}$. 
	\end{algorithmic}
\end{algorithm}

\begin{theorem}\label{th-main1}
Suppose the loss function $f(\cdot, \cdot)$ satisfies Assumption~\ref{assump-lip-continue}. If we set  $k=\left( (\frac{\alpha}{\epsilon_2})^{\frac{1}{z}} 3^{\left \lceil \log s \right \rceil}\cdot 2^{\frac{\left \lceil \log s \right \rceil^2+3\left \lceil \log s \right \rceil+8}{4}} \right)^{\rho}=\big( (\frac{\alpha}{\epsilon_2})^{\frac{1}{z}}  2^{O((\log s)^2)} \big)^{\rho}$ in Algorithm~\ref{alg: aggtree}, the returned  set $C_{\nu_0}$ with the weights $\{w_1, \cdots, w_k\}$ yields a $(\beta, \epsilon_2)_z$-coreset. The time complexity of constructing the tree $\mathcal{T}$ is  $O\left(s k^2 \Psi(N,s,d)+sk^3\right)$, where $\Psi(N,s,d)$ is the complexity of performing the counting each  time~\cite{abo2016faq} (for counting an acyclic join, $\Psi(N,s,d) = O\left(s d^{2} N \log (N)\right)$).
\end{theorem}

\begin{remark}
Note that our coreset size is quasi-polynomial since it has a factor $2^{O((\log s)^2)}$. But we would like to emphasize that this factor usually is small for most practical problems. For example,  in the  standard TPC-H benchmark~\cite{DBLP:journals/sigmod/PoessF00}, $s$ is no larger than $8$, although the design matrix over joining the $s$ tables can be quite large. 
\end{remark}

To prove the above theorem, we first introduce the following key lemma.
We let $r_0=\frac{\Delta}{k^{\frac{1}{\rho}}}$. From Claim~\ref{claim-doubling} we know that  the entire data set $P$ can be covered by $k$ balls with radius $r_0$.

\begin{lemma}\label{lemma-L}
For each $h=0, 1, \cdots, \left \lceil \log s \right \rceil$, $L_h\leq 3^{h}\cdot 2^{\frac{h^2+3h+8}{4}} r_0$. 
\end{lemma}
\begin{proof}
Since the Gonzalez's algorithm yields an approximation factor $2$, from Claim~\ref{claim-doubling} we know that $L_0\le 2r_0$. 
Then we consider the case $h\geq 1$. Suppose the algorithm is constructing the $h$-th level of $\mathcal{T}$. 
In Step~\ref{loop} we repeatedly select two nodes $\nu_i,\nu_j$ from the $(h-1)$-th level that form a parent node $\nu$ at the $h$-th level. Let  $C_{opt}$ be the optimal solution of the $k$-center clustering on  $\operatorname{Proj}_{H_{\nu_i} \times H_{\nu_j}} (P)$. Since $P$ can be covered by $k$ balls with radius $r_0$, so can the projection $\operatorname{Proj}_{H_{\nu_i} \times H_{\nu_j}} (P)$. That is, 
\begin{eqnarray}
\mathtt{d}(\operatorname{Proj}_{H_{\nu_i} \times H_{\nu_j}} (P), C_{opt}) \le r_0. \label{for-l1}
\end{eqnarray}
Also, since $\mathtt{d}(\operatorname{Proj}_{H_{\nu_i}} (P), C_{\nu_i}) \le L_{h-1}$ and $\mathtt{d}(\operatorname{Proj}_{H_{\nu_j}} (P), C_{\nu_j}) \le L_{h-1}$, we have
\begin{eqnarray}
\mathtt{d}(\operatorname{Proj}_{H_{\nu_i} \times H_{\nu_j}} (P),\widetilde{C_{\nu_i}\times C_{\nu_j}})\leq \sqrt{2}L_{h-1}.\label{for-l3}
\end{eqnarray}
Note that we use $\widetilde{C_{\nu_i}\times C_{\nu_j}}$ instead of $C_{\nu_i}\times C_{\nu_j}$ in the above (\ref{for-l3}). The reason is that the Cartesian product $C_{\nu_i}\times C_{\nu_j}$ may contain some ``empty'' grid points (see Algorithm~\ref{alg: aggtree} step 4(b)(iii)), and the distance bound ``$\sqrt{2}L_{h-1}$'' does not hold for them. 
%Since we have removed the empty grid points from $C_{\nu_i}\times C_{\nu_j}$ in step 4(b)(\rmnum{3}), we know that 
%that the pseudo-cube corresponding to each point in $\widetilde{C_{\nu_i}\times C_{\nu_j}}$ is not empty,

Together with (\ref{for-l1}), we know 
\begin{eqnarray}
\mathtt{d}(\widetilde{C_{\nu_i}\times C_{\nu_j}}, C_{opt}) \le r_0+\sqrt{2}L_{h-1}. \label{for-l2}
\end{eqnarray}
That is, if we run $k$-center clustering on $\widetilde{C_{\nu_i}\times C_{\nu_j}} $, the optimal radius should be no larger than $r_0+\sqrt{2}L_{h-1}$. So if we run the $2$-approximation Gonzalez's algorithm on $\widetilde{C_{\nu_i}\times C_{\nu_j}} $ and obtain the center set $C_\nu$, then 
\begin{eqnarray}
\mathtt{d}(\widetilde{C_{\nu_i}\times C_{\nu_j}}, C_{\nu})  \le 2(r_0 + \sqrt{2}L_{h-1}).\label{for-l4}
\end{eqnarray}
Further, we combine (\ref{for-l3}) and (\ref{for-l4}), and obtain 
\begin{eqnarray}
\mathtt{d}(\operatorname{Proj}_{H_{ \nu_i } \times H_{ \nu_j}} (P), C_{\nu}) \le 2(r_0 + \sqrt{2}L_{h-1}) + \sqrt{2}L_{h-1}.\label{for-l5}
\end{eqnarray}
For each $c_{\nu,l}\in C_\nu$, we construct a pseudo-cube
\begin{eqnarray}
 \prod_{t\in I_{\nu_i}\cup I_{\nu_j}} \mathbb{B}(\operatorname{Proj}_{\hat{D}_t}(c_{\nu,l}), 2(r_0 + \sqrt{2}L_{h-1}) + \sqrt{2}L_{h-1}).    
\end{eqnarray}
Then we know that $\operatorname{Proj}_{H_{ \nu_i} \times H_{ \nu_j}} (P) $ is covered by the union of the obtained $k$ pseudo-cubes. Because $|I_\nu|=|I_{\nu_i}\cup I_{\nu_j}|\le 2^h$, we have 
$L_{h} \le (2(r_0 + \sqrt{2}L_{h-1}) + \sqrt{2}L_{h-1})\cdot\sqrt{2^{h}}$. Together with $L_0\le 2r_0$, we can solve this recursion function and obtain $L_{h}\leq 3^{h}\cdot 2^{\frac{h^2+3h+8}{4}} r_0$.
%\begin{eqnarray}
%     L_{h}\leq 3^{h}\cdot 2^{\frac{h^2+3h+8}{4}} r_0.
%\end{eqnarray}
\end{proof}

Lemma~\ref{lemma-L} indicates that the difference between $C_{\nu_0}$ and $P$ is bounded in the space. To prove it is a qualified coreset with respect to the ERM problem (\ref{for-er}), we also need to assign a weight to each point of $C_{\nu_0}$. In Step~\ref{step} of Algorithm~\ref{alg: aggtree}, we set the weight $w_i=|P\cap \mathtt{PC}_{I_{\nu_0}}(c_{\nu_0, i}, L_{\left \lceil \log s \right \rceil}) |$ (for simplicity, we temporarily assume these pseudo-cubes are disjoint, and leave the discussion on the overlap issue to Section~\ref{weighting}). The detailed proof of Theorem~\ref{th-main1} is placed to our appendix. 

\begin{remark}
In a real implementation, if the coreset size  is given (e.g., $5\%$  of the input data size $|P|$), we can directly set the value for $k$ to run Algorithm~\ref{alg: aggtree}. 
Another scenario is that we are given a restricted variance between $P$ and the output $C_{\nu_0}$ (i.e., the  difference $\mathtt{d}(P, C_{\nu_0})$ is required to be no larger than a threshold). Then we can try the value for $k$ via doubling search. For example, starting from a small constant $k_0$, we can try $k=k_0$, $2k_0$, $2^2 k_0$, $\cdots$, until the difference is lower than the threshold. 
\end{remark}

\subsection{The Overlap Issue}\label{weighting}

We consider the $k$ pseudo-cubes $\mathtt{PC}_{I_{\nu_0}}(c_{\nu_0, i}, L_{\left \lceil \log s \right \rceil})$, $1\leq i\leq k$, obtained from Algorithm~\ref{alg: aggtree}. Since $\nu_0$ is the root, it is easy to know $I_{\nu_0}=[s]$ and $H_{\nu_0}=\mathbb{R}^d$. For the sake of simplicity, we use $\mathtt{PC}_i$ to denote the pseudo-cube $\mathtt{PC}_{I_{\nu_0}}(c_{\nu_0, i}, L_{\left \lceil \log s \right \rceil})$ and $L$ to denote $L_{\left \lceil \log s \right \rceil}$ below. 

If these pseudo-cubes are disjoint, we can simply set the weight $w_i=|P\cap \mathtt{PC}_i |$ as Step~\ref{step} of Algorithm~\ref{alg: aggtree}. However, these pseudo-cubes may have overlaps and such an assignment for the weights cannot guarantee the correctness of our coreset. For example, if a point $q\in P$ is covered by two pseudo-cubes $\mathtt{PC}_{i_1}$ and $\mathtt{PC}_{i_2}$, we should assign $q$ to only one pseudo-cube. Note that this overlap issue is trivial if the whole data matrix $P$ is available. But it can be troublesome  for relational data, because it is quite inefficient to perform the counting on their union $P\cap \big(\mathtt{PC}_{i_1}\cup\mathtt{PC}_{i_2}\big)$.  Moreover, the $k$ pseudo-cubes can partition the space into as large as $2^k$ different regions, and it is challenging to deal with so many overlaps. 

Fortunately, we can solve this issue by using random sampling. The key observation is that we can tolerate a small error on each weight $w_i$. Let $\delta\in (0,1)$. If we obtain a set of approximate weights $W'=\{w'_1, \cdots, w'_k\}$ that satisfy $w'_i\in (1\pm\delta)w_i$ for $1\leq i\leq k$, then we have 
\begin{eqnarray}
 \sum^k_{i=1}w'_i f(\theta, c_{\nu_0,i})\in (1\pm\delta)\sum^k_{i=1}w_i f(\theta, c_{\nu_0,i}),  \label{for-weightapprox}
\end{eqnarray}
for any $\theta$ in the hypothesis space $\mathbb{H}$. We can use the following idea to obtain a qualified $W'$. 

\textbf{High-level idea.} Without loss of generality, we assume all the $k$ pseudo-cubes are not empty (otherwise, we can directly remove the empty pseudo-cubes). Then we consider the pseudo-cubes one by one. For $\mathtt{PC}_1$, we directly set $w'_1=|P\cap \mathtt{PC}_1|$. Suppose currently we have already obtained the values $w'_1, \cdots, w'_{i_0}$, and try to determine the value for $w'_{i_0+1}$. We take a uniform sample of $m$ points from $P\cap \mathtt{PC}_{i_0+1}$ by using the sampling technique for relational data~\cite{zhao2018random}. Each sampled point corresponds a binary random variable $x$: if it belongs to $\mathtt{PC}_{i_0+1}\setminus(\cup^{i_0}_{i=1}\mathtt{PC}_{i})$, $x=1$; otherwise, $x=0$. If $m$ is sufficiently large, from the Chernoff bound, we can prove that the sum of these $m$ random variables (denote by $g$) over $m$ can serve as a good estimation of  
\begin{eqnarray}
\tau_{i_0+1}=\frac{|P\cap (\mathtt{PC}_{i_0+1}\setminus(\cup^{i_0}_{i=1}\mathtt{PC}_{i}))|}{|P\cap \mathtt{PC}_{i_0+1}|}.
\end{eqnarray}
 Thus we can set $w'_{i_0+1}=\frac{g}{m}\cdot |P\cap \mathtt{PC}_{i_0+1}|\approx |P\cap (\mathtt{PC}_{i_0+1}\setminus(\cup^{i_0}_{i=1}\mathtt{PC}_{i}))|$. 

A remaining issue of the above method is that the ratio $\tau_{i_0+1}$ can be extremely small, that is, we have to set $m$ to be too large for guaranteeing the multiplicative ``$1\pm \delta$'' error bound. Our idea for solving this issue is from the geometry. If  $\tau_{i_0+1}$ is extremely small, we label  $\mathtt{PC}_{i_0+1}$ as a ``light'' pseudo-cube. Meanwhile, it should has at least one ``heavy'' neighbor from $\{\mathtt{PC}_{1}, \cdots, \mathtt{PC}_{i_0}\}$. The reason is that a small $\tau_{i_0+1}$ implies that there must exist some $1\leq i_1\leq i_0$, such that the intersection $P\cap \mathtt{PC}_{i_0+1}\cap \mathtt{PC}_{i_1}$ takes a significant part of $P\cap \mathtt{PC}_{i_0+1}$. And we label $\mathtt{PC}_{i_1}$ as a ``heavy'' neighbor of $\mathtt{PC}_{i_0+1}$. Moreover, the distance between $\mathtt{PC}_{i_1}$ and $\mathtt{PC}_{i_0+1}$ should be bounded (since they have overlap in the space). Therefore, we can just ignore $\mathtt{PC}_{i_0+1}$ and use $\mathtt{PC}_{i_1}$ to represent their union $\mathtt{PC}_{i_1}\cup\mathtt{PC}_{i_0+1}$. This idea can help us to avoid taking a too large sample for light pseudo-cube. Overall, we have the following theorem, and the detailed proof is placed to our appendix. 

\begin{theorem}
\label{th-main2}
Suppose the loss function $f(\cdot, \cdot)$ satisfies Assumption~\ref{assump-lip-continue} and $\epsilon_1>\beta$. If we set $k$ in Algorithm~\ref{alg: aggtree} as Theorem~\ref{th-main1}, and let $m \ge \Theta(\frac{k^2}{(\epsilon_1-\beta)^2\epsilon_1}\log\frac{k}{\lambda})$ where $\lambda\in (0,1)$, the returned set $C_{\nu_0}$ with the weights $W'$ yields a $(\epsilon_1, \epsilon_2)_z$-coreset with probability $1-\lambda$. The sampling procedure takes $O(k^2md)$ time.
\end{theorem}

\section{Experimental Evaluation}

We evaluate the performance of our relational coreset on three popular machine learning problems, the SVM with soft margin ($\alpha=O(\|\theta\|_{2})$, $\beta=0$, and $z=1$), the $k_c$-means clustering\footnote{We use ``$k_c$'' instead of ``$k$'' to avoid being confused with the coreset size $k$} ($\alpha=O(\frac{1}{\epsilon})$, $\beta=\epsilon$, and $z=2$, where $\epsilon$ can be any small number in $(0,1)$), and the logistic regression ($\alpha=O(\|\theta\|_{2})$, $\beta=0$, and $z=1$). All the experimental results were obtained on a server equipped with 3.0GHz Intel CPUs and 384GB main memory. Our algorithms were implemented in Python with PostgreSQL 12.10. We release our codes at Github~\cite{Github}.

\textbf{Data sets and Queries.} We design four different join queries ($\mathrm{Q1}$-$\mathrm{Q4}$) on three real relational data sets. $\mathrm{Q1}$ and $\mathrm{Q2}$ are designed on a labeled data set \textsc{Home Credit~\cite{HomeCredit}}, and we use them to solve the SVM and logistic regression problems. $\mathrm{Q3}$ and $\mathrm{Q4}$ are foreign key joins~\cite{DBLP:conf/sigmod/AcharyaGPR99a} designed on the unlabeled data sets \textsc{Yelp~\cite{Yelp}} and \textsc{Favorita~\cite{Favorita}} respectively, and we use them to solve the $k_c$-means clustering problem.

\textbf{Baseline methods.} 
We consider five baseline methods for comparison. 
(1) \textsc{Original}: construct the complete design matrix $P$ and run the training algorithm directly on it. 
(2) \textsc{Ori-Gon}: construct $P$ as \textsc{Original},  run the Gonzalez's algorithm~\cite{DBLP:journals/tcs/Gonzalez85} on $P$ to obtain the centers, and then run the training algorithm on the centers. 
(3) \textsc{Uniform}: the relational uniform sampling algorithm~\cite{zhao2018random}. 
(4) \textsc{R$k$-means}: the relational $k_c$-means algorithm~\cite{DBLP:conf/aistats/CurtinM0NOS20}. 
%It  performs the $\kappa$-means ($\kappa\in(0,k_c]$) on each table and then constructs a grid coreset of size $\kappa^s$.  
(5) \textsc{RCore}: our relational coreset approach.

We consider both the running time and optimization quality. We record the end-to-end runtime that includes the design matrix/coreset construction time and the training time. 
For the optimization quality, we take the objective value $F(\theta^{*})$ obtained by  \textsc{Original} as the optimal objective value; for each baseline method, we define  ``$\textbf{Approx.}$''$=\frac{ F(\theta)-F(\theta^{*}) }{F(\theta^{*})}$, where $F(\theta)$ is the objective value obtained by the method. 

\renewcommand{\arraystretch}{1.1}
\setlength\tabcolsep{4pt} 
\begin{table}[!h t b p]
\centering
\begin{tabular}{c|c||c c c c c} 
\toprule[2pt]
\multicolumn{2}{c||}{\textbf{Coreset size}} & $200$ & $400$ & $600$ & $800$ & $1000$\\
\hline
\multirow{4}*{\textbf{\shortstack{End-to-end \\runtime (s)}}}
& \textsc{Original} & \multicolumn{5}{c}{$>21600$}\\
\cline{2-7}
& \textsc{Ori-Gon} & $3808$ & $5208$ & $6606$ & $8044$ & $9434$\\
\cline{2-7}
& \textsc{Uniform} & $34$ & $35$ & $35$ & $36$ & $38$\\
\cline{2-7}
& \textsc{RCore} & $208$ & $288$ & $363$ & $446$ & $531$\\
\hline
\multirow{3}*{\textbf{Approx.}} 
& \textsc{Ori-Gon} & $1.41$ & $1.50$ & $1.29$ & $1.12$ & $0.95$\\
\cline{2-7}
& \textsc{Uniform} & $2.22$ & $2.60$ & $2.09$ & $2.23$ & $2.14$\\
\cline{2-7}
& \textsc{RCore} & $0.92$ & $0.31$ & $0.27$ & $0.16$ & $0.02$\\
\toprule[2pt]
\end{tabular}
\caption{The results of SVM on $\mathrm{Q1}$.}
\label{svmQ1}
%\vspace{-0.2in}
\end{table}
\begin{figure}[ht]
    \centering
    \includegraphics[width=1\linewidth,height=6cm]{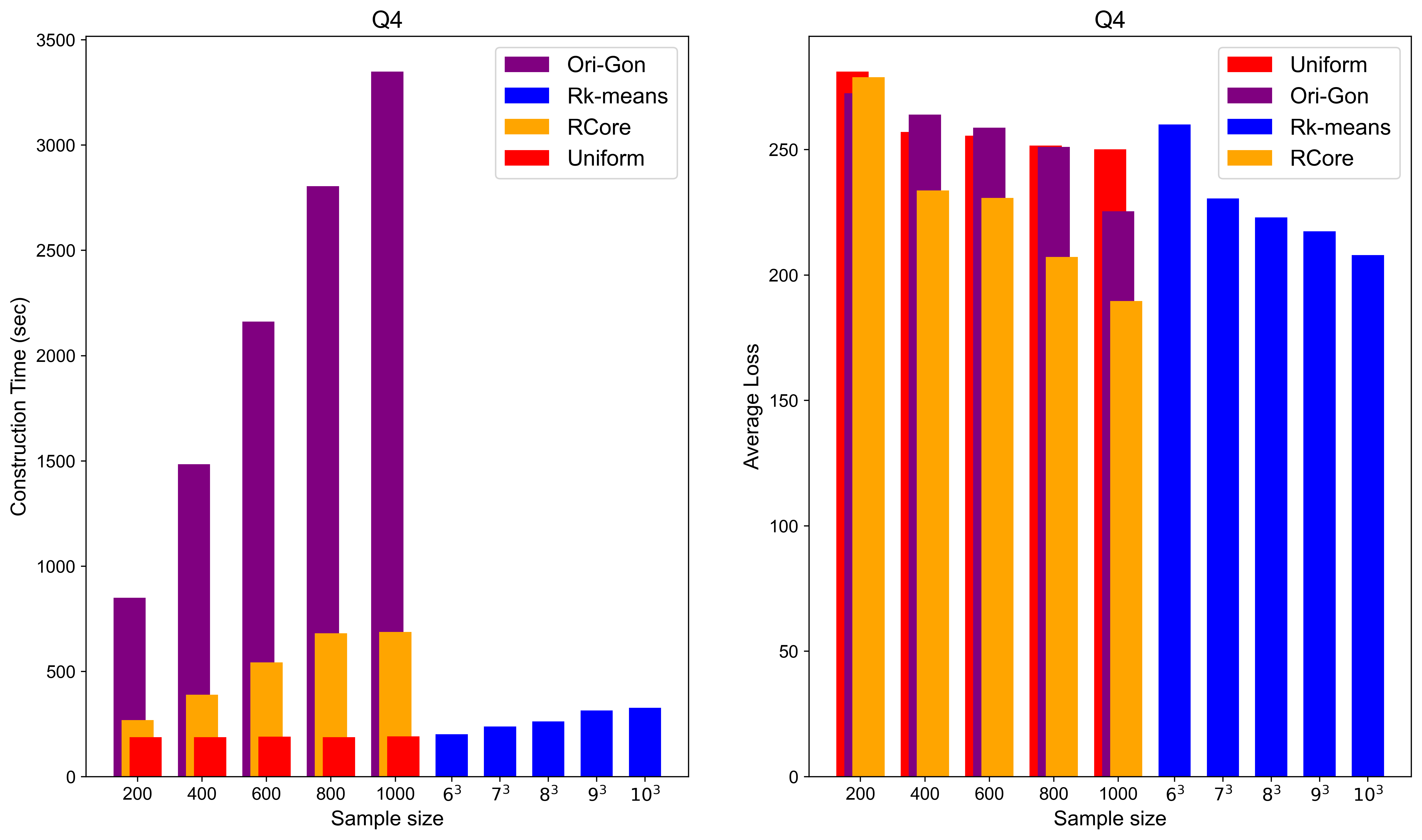}
      \caption{The results of $k_c$-means on $\mathrm{Q4}$.}
  \label{fig-show}
%  \vspace{-0.1in}
\end{figure}
Due to the space limit, we only present part of the results (Table~\ref{svmQ1} and Figure~\ref{fig-show}) here. In general, \textsc{Original} and \textsc{Ori-Gon} are the most time-consuming ones because they need to construct the whole design matrix and perform the operations (like train the models or run the Gonzalez's algorithm) on the data (even if the queries are foreign key joins). 
%{\color{red}There won't be much relief even if queries are foreign key joins and thus $n=O(N)$~\cite{DBLP:conf/sigmod/AcharyaGPR99a}.}
\textsc{Uniform} is always the fastest one, because it only takes the simple uniform sampling procedure; but its overall optimization performance is relatively poor. Except \textsc{Original}, our \textsc{RCore} has the best optimization quality for most cases with acceptable running time.
We refer the reader to our appendix for the detailed  experimental results.

\section{Conclusion}
\label{sec-app}
In this paper, we propose a novel relational coreset method based on  the aggregation tree with pseudo-cube technique. Our method can be applied to several popular machine learning models and
has provable quality guarantees. 
In future, there are also several interesting problems deserved to study, e.g., the relational coresets construction for more complicated machine learning models, and the privacy-preserving issues for relational coresets when the input data contains sensitive attributes.

\section{Acknowledgements}
The research of this work was supported in part by National Key R\&D program of China through grant 2021YFA1000900 and the Provincial NSF of Anhui through grant 2208085MF163. We also want to thank the anonymous reviewers for their helpful comments.

\bibliographystyle{plain}
\bibliography{ref}

\begin{thebibliography}{10}

\bibitem{abo2021approximate}
Mahmoud Abo-Khamis, Sungjin Im, Benjamin Moseley, Kirk Pruhs, and Alireza
  Samadian.
\newblock Approximate aggregate queries under additive inequalities.
\newblock In {\em Symposium on Algorithmic Principles of Computer Systems
  (APOCS)}, pages 85--99. SIAM, 2021.

\bibitem{abo2021relational}
Mahmoud Abo-Khamis, Sungjin Im, Benjamin Moseley, Kirk Pruhs, and Alireza
  Samadian.
\newblock A relational gradient descent algorithm for support vector machine
  training.
\newblock In {\em Symposium on Algorithmic Principles of Computer Systems
  (APOCS)}, pages 100--113. SIAM, 2021.

\bibitem{abo2018database}
Mahmoud Abo~Khamis, Hung~Q Ngo, XuanLong Nguyen, Dan Olteanu, and Maximilian
  Schleich.
\newblock In-database learning with sparse tensors.
\newblock In {\em Proceedings of the 37th ACM SIGMOD-SIGACT-SIGAI Symposium on
  Principles of Database Systems}, pages 325--340, 2018.

\bibitem{abo2016faq}
Mahmoud Abo~Khamis, Hung~Q Ngo, and Atri Rudra.
\newblock Faq: questions asked frequently.
\newblock In {\em Proceedings of the 35th ACM SIGMOD-SIGACT-SIGAI Symposium on
  Principles of Database Systems}, pages 13--28, 2016.

\bibitem{DBLP:conf/sigmod/AcharyaGPR99a}
Swarup Acharya, Phillip~B. Gibbons, Viswanath Poosala, and Sridhar Ramaswamy.
\newblock Join synopses for approximate query answering.
\newblock In Alex Delis, Christos Faloutsos, and Shahram Ghandeharizadeh,
  editors, {\em {SIGMOD} 1999, Proceedings {ACM} {SIGMOD} International
  Conference on Management of Data, June 1-3, 1999, Philadelphia, Pennsylvania,
  {USA}}, pages 275--286. {ACM} Press, 1999.

\bibitem{Arthur2007kmeansTA}
David Arthur and Sergei Vassilvitskii.
\newblock k-means++: the advantages of careful seeding.
\newblock In {\em SODA '07}, 2007.

\bibitem{DBLP:journals/siamcomp/AtseriasGM13}
Albert Atserias, Martin Grohe, and D{\'{a}}niel Marx.
\newblock Size bounds and query plans for relational joins.
\newblock {\em {SIAM} J. Comput.}, 42(4):1737--1767, 2013.

\bibitem{DBLP:conf/icml/BachemLH017}
Olivier Bachem, Mario Lucic, S.~Hamed Hassani, and Andreas Krause.
\newblock Uniform deviation bounds for k-means clustering.
\newblock In Doina Precup and Yee~Whye Teh, editors, {\em Proceedings of the
  34th International Conference on Machine Learning, {ICML} 2017, Sydney, NSW,
  Australia, 6-11 August 2017}, volume~70 of {\em Proceedings of Machine
  Learning Research}, pages 283--291. {PMLR}, 2017.

\bibitem{DBLP:conf/kdd/BachemL018}
Olivier Bachem, Mario Lucic, and Andreas Krause.
\newblock Scalable k -means clustering via lightweight coresets.
\newblock In Yike Guo and Faisal Farooq, editors, {\em Proceedings of the 24th
  {ACM} {SIGKDD} International Conference on Knowledge Discovery {\&} Data
  Mining, {KDD} 2018, London, UK, August 19-23, 2018}, pages 1119--1127. {ACM},
  2018.

\bibitem{belkin2003problems}
Mikhail Belkin.
\newblock Problems of learning on manifolds, 2003.

\bibitem{borsos2020coresets}
Zal{\'a}n Borsos, Mojmir Mutny, and Andreas Krause.
\newblock Coresets via bilevel optimization for continual learning and
  streaming.
\newblock {\em Advances in Neural Information Processing Systems},
  33:14879--14890, 2020.

\bibitem{chan2016hierarchical}
T-H~Hubert Chan, Anupam Gupta, Bruce~M Maggs, and Shuheng Zhou.
\newblock On hierarchical routing in doubling metrics.
\newblock {\em ACM Transactions on Algorithms (TALG)}, 12(4):1--22, 2016.

\bibitem{DBLP:conf/sigmod/ChaudhuriMN99}
Surajit Chaudhuri, Rajeev Motwani, and Vivek~R. Narasayya.
\newblock On random sampling over joins.
\newblock In Alex Delis, Christos Faloutsos, and Shahram Ghandeharizadeh,
  editors, {\em {SIGMOD} 1999, Proceedings {ACM} {SIGMOD} International
  Conference on Management of Data, June 1-3, 1999, Philadelphia, Pennsylvania,
  {USA}}, pages 263--274. {ACM} Press, 1999.

\bibitem{chen2020random}
Yu~Chen and Ke~Yi.
\newblock Random sampling and size estimation over cyclic joins.
\newblock In {\em 23rd International Conference on Database Theory (ICDT
  2020)}. Schloss Dagstuhl-Leibniz-Zentrum f{\"u}r Informatik, 2020.

\bibitem{cheng2019nonlinear}
Zhaoyue Cheng and Nick Koudas.
\newblock Nonlinear models over normalized data.
\newblock In {\em 2019 IEEE 35th International Conference on Data Engineering
  (ICDE)}, pages 1574--1577. IEEE, 2019.

\bibitem{DBLP:conf/icde/ChengKZ021}
Zhaoyue Cheng, Nick Koudas, Zhe Zhang, and Xiaohui Yu.
\newblock Efficient construction of nonlinear models over normalized data.
\newblock In {\em 37th {IEEE} International Conference on Data Engineering,
  {ICDE} 2021, Chania, Greece, April 19-22, 2021}, pages 1140--1151. {IEEE},
  2021.

\bibitem{DBLP:journals/cacm/Codd83}
E.~F. Codd.
\newblock A relational model of data for large shared data banks (reprint).
\newblock {\em Commun. {ACM}}, 26(1):64--69, 1983.

\bibitem{DBLP:journals/jacm/Cohen-AddadFS21}
Vincent Cohen{-}Addad, Andreas~Emil Feldmann, and David Saulpic.
\newblock Near-linear time approximation schemes for clustering in doubling
  metrics.
\newblock {\em J. {ACM}}, 68(6):44:1--44:34, 2021.

\bibitem{DBLP:conf/stoc/Cohen-AddadSS21}
Vincent Cohen{-}Addad, David Saulpic, and Chris Schwiegelshohn.
\newblock A new coreset framework for clustering.
\newblock In Samir Khuller and Virginia~Vassilevska Williams, editors, {\em
  {STOC} '21: 53rd Annual {ACM} {SIGACT} Symposium on Theory of Computing,
  Virtual Event, Italy, June 21-25, 2021}, pages 169--182. {ACM}, 2021.

\bibitem{DBLP:conf/iclr/ColemanYMMBLLZ20}
Cody Coleman, Christopher Yeh, Stephen Mussmann, Baharan Mirzasoleiman, Peter
  Bailis, Percy Liang, Jure Leskovec, and Matei Zaharia.
\newblock Selection via proxy: Efficient data selection for deep learning.
\newblock In {\em 8th International Conference on Learning Representations,
  {ICLR} 2020, Addis Ababa, Ethiopia, April 26-30, 2020}. OpenReview.net, 2020.

\bibitem{cortes1995support}
Corinna Cortes and Vladimir Vapnik.
\newblock Support-vector networks.
\newblock {\em Machine learning}, 20(3):273--297, 1995.

\bibitem{cramer2004early}
Jan~Salomon Cramer.
\newblock The early origins of the logit model.
\newblock {\em Studies in History and Philosophy of Science Part C: Studies in
  History and Philosophy of Biological and Biomedical Sciences},
  35(4):613--626, 2004.

\bibitem{DBLP:conf/aistats/CurtinM0NOS20}
Ryan~R. Curtin, Benjamin Moseley, Hung~Q. Ngo, XuanLong Nguyen, Dan Olteanu,
  and Maximilian Schleich.
\newblock Rk-means: Fast clustering for relational data.
\newblock In Silvia Chiappa and Roberto Calandra, editors, {\em The 23rd
  International Conference on Artificial Intelligence and Statistics, {AISTATS}
  2020, 26-28 August 2020, Online [Palermo, Sicily, Italy]}, volume 108 of {\em
  Proceedings of Machine Learning Research}, pages 2742--2752. {PMLR}, 2020.

\bibitem{DBLP:books/lib/BergCKO08}
Mark de~Berg, Otfried Cheong, Marc~J. van Kreveld, and Mark~H. Overmars.
\newblock {\em Computational geometry: algorithms and applications, 3rd
  Edition}.
\newblock Springer, 2008.

\bibitem{DBLP:conf/icml/DingLHL16}
Hu~Ding, Yu~Liu, Lingxiao Huang, and Jian Li.
\newblock K-means clustering with distributed dimensions.
\newblock In Maria{-}Florina Balcan and Kilian~Q. Weinberger, editors, {\em
  Proceedings of the 33nd International Conference on Machine Learning, {ICML}
  2016, New York City, NY, USA, June 19-24, 2016}, volume~48 of {\em {JMLR}
  Workshop and Conference Proceedings}, pages 1339--1348. JMLR.org, 2016.

\bibitem{DBLP:conf/stoc/Dyer03}
Martin~E. Dyer.
\newblock Approximate counting by dynamic programming.
\newblock In Lawrence~L. Larmore and Michel~X. Goemans, editors, {\em
  Proceedings of the 35th Annual {ACM} Symposium on Theory of Computing, June
  9-11, 2003, San Diego, CA, {USA}}, pages 693--699. {ACM}, 2003.

\bibitem{Favorita}
Favorita dataset.
\newblock \url{https://www.kaggle.com/c/favorita-grocery-sales-forecasting}.

\bibitem{DBLP:journals/widm/Feldman20}
Dan Feldman.
\newblock Core-sets: An updated survey.
\newblock {\em WIREs Data Mining Knowl. Discov.}, 10(1), 2020.

\bibitem{Github}
Source code of our relational coreset approach.
\newblock
  \url{https://github.com/cjx-zar/Coresets-for-Relational-Data-and-The-Applications}.

\bibitem{DBLP:journals/tcs/Gonzalez85}
Teofilo~F. Gonzalez.
\newblock Clustering to minimize the maximum intercluster distance.
\newblock {\em Theor. Comput. Sci.}, 38:293--306, 1985.

\bibitem{har2006fast}
Sariel Har-Peled and Manor Mendel.
\newblock Fast construction of nets in low-dimensional metrics and their
  applications.
\newblock {\em SIAM Journal on Computing}, 35(5):1148--1184, 2006.

\bibitem{HomeCredit}
Home credit dataset.
\newblock \url{https://tianchi.aliyun.com/dataset/dataDetail?dataId=89722}.

\bibitem{DBLP:conf/focs/HuangJLW18}
Lingxiao Huang, Shaofeng~H.{-}C. Jiang, Jian Li, and Xuan Wu.
\newblock Epsilon-coresets for clustering (with outliers) in doubling metrics.
\newblock In Mikkel Thorup, editor, {\em 59th {IEEE} Annual Symposium on
  Foundations of Computer Science, {FOCS} 2018, Paris, France, October 7-9,
  2018}, pages 814--825. {IEEE} Computer Society, 2018.

\bibitem{DBLP:conf/nips/HugginsCB16}
Jonathan~H. Huggins, Trevor Campbell, and Tamara Broderick.
\newblock Coresets for scalable bayesian logistic regression.
\newblock In Daniel~D. Lee, Masashi Sugiyama, Ulrike von Luxburg, Isabelle
  Guyon, and Roman Garnett, editors, {\em Advances in Neural Information
  Processing Systems 29: Annual Conference on Neural Information Processing
  Systems 2016, December 5-10, 2016, Barcelona, Spain}, pages 4080--4088, 2016.

\bibitem{DBLP:conf/pods/IndykMMM14}
Piotr Indyk, Sepideh Mahabadi, Mohammad Mahdian, and Vahab~S. Mirrokni.
\newblock Composable core-sets for diversity and coverage maximization.
\newblock In Richard Hull and Martin Grohe, editors, {\em Proceedings of the
  33rd {ACM} {SIGMOD-SIGACT-SIGART} Symposium on Principles of Database
  Systems, PODS'14, Snowbird, UT, USA, June 22-27, 2014}, pages 100--108.
  {ACM}, 2014.

\bibitem{DBLP:journals/tods/KhamisCMNNOS20}
Mahmoud~Abo Khamis, Ryan~R. Curtin, Benjamin Moseley, Hung~Q. Ngo, XuanLong
  Nguyen, Dan Olteanu, and Maximilian Schleich.
\newblock Functional aggregate queries with additive inequalities.
\newblock {\em {ACM} Trans. Database Syst.}, 45(4):17:1--17:41, 2020.

\bibitem{DBLP:conf/apocs/KhamisIMPS21a}
Mahmoud~Abo Khamis, Sungjin Im, Benjamin Moseley, Kirk Pruhs, and Alireza
  Samadian.
\newblock A relational gradient descent algorithm for support vector machine
  training.
\newblock In Michael Schapira, editor, {\em 2nd Symposium on Algorithmic
  Principles of Computer Systems, {APOCS} 2020, Virtual Conference, January 13,
  2021}, pages 100--113. {SIAM}, 2021.

\bibitem{DBLP:conf/pods/Khamis0NOS18}
Mahmoud~Abo Khamis, Hung~Q. Ngo, XuanLong Nguyen, Dan Olteanu, and Maximilian
  Schleich.
\newblock In-database learning with sparse tensors.
\newblock In Jan~Van den Bussche and Marcelo Arenas, editors, {\em Proceedings
  of the 37th {ACM} {SIGMOD-SIGACT-SIGAI} Symposium on Principles of Database
  Systems, Houston, TX, USA, June 10-15, 2018}, pages 325--340. {ACM}, 2018.

\bibitem{li2001improved}
Yi~Li, Philip~M Long, and Aravind Srinivasan.
\newblock Improved bounds on the sample complexity of learning.
\newblock {\em Journal of Computer and System Sciences}, 62(3):516--527, 2001.

\bibitem{DBLP:journals/jmlr/LucicFKF17}
Mario Lucic, Matthew Faulkner, Andreas Krause, and Dan Feldman.
\newblock Training gaussian mixture models at scale via coresets.
\newblock {\em J. Mach. Learn. Res.}, 18:160:1--160:25, 2017.

\bibitem{luukkainen1998assouad}
Jouni Luukkainen.
\newblock Assouad dimension: antifractal metrization, porous sets, and
  homogeneous measures.
\newblock {\em Journal of the Korean Mathematical Society}, 35(1):23--76, 1998.

\bibitem{DBLP:journals/jacm/Marx13}
D{\'{a}}niel Marx.
\newblock Tractable hypergraph properties for constraint satisfaction and
  conjunctive queries.
\newblock {\em J. {ACM}}, 60(6):42:1--42:51, 2013.

\bibitem{moseley2021relational}
Benjamin Moseley, Kirk Pruhs, Alireza Samadian, and Yuyan Wang.
\newblock Relational algorithms for k-means clustering.
\newblock In {\em 48th International Colloquium on Automata, Languages, and
  Programming (ICALP 2021)}. Schloss Dagstuhl-Leibniz-Zentrum f{\"u}r
  Informatik, 2021.

\bibitem{DBLP:journals/ki/MunteanuS18}
Alexander Munteanu and Chris Schwiegelshohn.
\newblock Coresets-methods and history: {A} theoreticians design pattern for
  approximation and streaming algorithms.
\newblock {\em K{\"{u}}nstliche Intell.}, 32(1):37--53, 2018.

\bibitem{DBLP:conf/gi/MunteanuSSW19}
Alexander Munteanu, Chris Schwiegelshohn, Christian Sohler, and David~P.
  Woodruff.
\newblock On coresets for logistic regression.
\newblock In Klaus David, Kurt Geihs, Martin Lange, and Gerd Stumme, editors,
  {\em 49. Jahrestagung der Gesellschaft f{\"{u}}r Informatik, 50 Jahre
  Gesellschaft f{\"{u}}r Informatik - Informatik f{\"{u}}r Gesellschaft,
  {INFORMATIK} 2019, Kassel, Germany, September 23-26, 2019}, volume {P-294} of
  {\em {LNI}}, pages 267--268. {GI}, 2019.

\bibitem{DBLP:conf/amw/0001NOS17}
Hung~Q. Ngo, XuanLong Nguyen, Dan Olteanu, and Maximilian Schleich.
\newblock In-database factorized learning.
\newblock In Juan~L. Reutter and Divesh Srivastava, editors, {\em Proceedings
  of the 11th Alberto Mendelzon International Workshop on Foundations of Data
  Management and the Web, Montevideo, Uruguay, June 7-9, 2017}, volume 1912 of
  {\em {CEUR} Workshop Proceedings}. CEUR-WS.org, 2017.

\bibitem{DBLP:journals/corr/Phillips16}
Jeff~M. Phillips.
\newblock Coresets and sketches.
\newblock {\em CoRR}, abs/1601.00617, 2016.

\bibitem{DBLP:journals/sigmod/PoessF00}
Meikel P{\"{o}}ss and Chris Floyd.
\newblock New {TPC} benchmarks for decision support and web commerce.
\newblock {\em {SIGMOD} Rec.}, 29(4):64--71, 2000.

\bibitem{market}
Report.
\newblock {\em Relational Database Management System Market: Industry Analysis
  and Forecast 2021-2027: By Type, Deployment, End users, and Region}.
\newblock 2022.

\bibitem{samadian2020unconditional}
Alireza Samadian, Kirk Pruhs, Benjamin Moseley, Sungjin Im, and Ryan Curtin.
\newblock Unconditional coresets for regularized loss minimization.
\newblock In {\em International Conference on Artificial Intelligence and
  Statistics}, pages 482--492. PMLR, 2020.

\bibitem{schleich2016learning}
Maximilian Schleich, Dan Olteanu, and Radu Ciucanu.
\newblock Learning linear regression models over factorized joins.
\newblock In {\em Proceedings of the 2016 International Conference on
  Management of Data}, pages 3--18, 2016.

\bibitem{DBLP:journals/corr/abs-1911-06577}
Maximilian Schleich, Dan Olteanu, Mahmoud~Abo Khamis, Hung~Q. Ngo, and XuanLong
  Nguyen.
\newblock Learning models over relational data: {A} brief tutorial.
\newblock {\em CoRR}, abs/1911.06577, 2019.

\bibitem{sener2018active}
Ozan Sener and Silvio Savarese.
\newblock Active learning for convolutional neural networks: A core-set
  approach.
\newblock In {\em International Conference on Learning Representations}, 2018.

\bibitem{sra2012optimization}
Suvrit Sra, Sebastian Nowozin, and Stephen~J Wright.
\newblock {\em Optimization for machine learning}.
\newblock Mit Press, 2012.

\bibitem{sumathi2007fundamentals}
Sai Sumathi and S~Esakkirajan.
\newblock {\em Fundamentals of relational database management systems},
  volume~47.
\newblock Springer, 2007.

\bibitem{DBLP:conf/nips/TelgarskyD13}
Matus Telgarsky and Sanjoy Dasgupta.
\newblock Moment-based uniform deviation bounds for \emph{k}-means and friends.
\newblock In Christopher J.~C. Burges, L{\'{e}}on Bottou, Zoubin Ghahramani,
  and Kilian~Q. Weinberger, editors, {\em Advances in Neural Information
  Processing Systems 26: 27th Annual Conference on Neural Information
  Processing Systems 2013. Proceedings of a meeting held December 5-8, 2013,
  Lake Tahoe, Nevada, United States}, pages 2940--2948, 2013.

\bibitem{DBLP:conf/nips/Vapnik91}
Vladimir Vapnik.
\newblock Principles of risk minimization for learning theory.
\newblock In John~E. Moody, Stephen~Jose Hanson, and Richard Lippmann, editors,
  {\em Advances in Neural Information Processing Systems 4, {[NIPS} Conference,
  Denver, Colorado, USA, December 2-5, 1991]}, pages 831--838. Morgan Kaufmann,
  1991.

\bibitem{yang2020towards}
Keyu Yang, Yunjun Gao, Lei Liang, Bin Yao, Shiting Wen, and Gang Chen.
\newblock Towards factorized svm with gaussian kernels over normalized data.
\newblock In {\em 2020 IEEE 36th International Conference on Data Engineering
  (ICDE)}, pages 1453--1464. IEEE, 2020.

\bibitem{Yelp}
Yelp dataset.
\newblock \url{https://www.kaggle.com/datasets/yelp-dataset/yelp-dataset}.

\bibitem{zhao2018random}
Zhuoyue Zhao, Robert Christensen, Feifei Li, Xiao Hu, and Ke~Yi.
\newblock Random sampling over joins revisited.
\newblock In {\em Proceedings of the 2018 International Conference on
  Management of Data}, pages 1525--1539, 2018.

\end{thebibliography}

\appendix

\section{The  Proof of Theorem~\ref{th-main1}}\label{gloalsize-prf}
For the sake of simplicity, we use $\mathtt{PC}_i$ to denote the pseudo-cube $\mathtt{PC}_{I_{\nu_0}}(c_{\nu_0, i}, L_{\left \lceil \log s \right \rceil})$ and $L$ to denote $L_{\left \lceil \log s \right \rceil}$ below.
Suppose there is no overlap between the $\mathtt{PC}_j$s.  We set the weight $w_j=|P\cap \mathtt{PC}_{j}|$ for each $1\leq j\leq k$, and each center $c_{\nu_0,j}\in C_{\nu_0}$ is a representative for the set  $P\cap\mathtt{PC}_j$ (we can also view each $c_{\nu_0,j}$ as a set of $w_j$ overlapping points in the space).  Since we assume the pseudo-cubes are disjoint, we have $\sum^k_{j=1}w_j=n$. 
Through Assumption~\ref{assump-lip-continue}, we have
\begin{eqnarray}\label{eq:obj}
     && n\left|\tilde F(\theta)-F(\theta)\right|\nonumber\\
      &=&\left| \sum_{c_{\nu_0,j}\in C_{\nu_0}} w_j f(\theta,c_{\nu_0,j}) - \sum_{p_i\in P}f(\theta,x_i) \right|\nonumber\\
      &\leq&\sum^k_{j=1}\sum_{p_i\in P\cap \mathtt{PC}_{j}} \left|   f(\theta,c_{\nu_0,j}) - f(\theta,p_i) \right|\nonumber\\
      &\le& n \alpha L^z + n \beta F(\theta).
\end{eqnarray}

For a given $\epsilon_2$,
through Claim~\ref{claim-doubling}, if we set $k = |C_{\nu_0}| = \left( (\frac{\alpha}{\epsilon_2})^{\frac{1}{z}} 3^{\left \lceil \log s \right \rceil}\cdot 2^{\frac{\left \lceil \log s \right \rceil^2+3\left \lceil \log s \right \rceil+8}{4}} \right)^{\rho}$, we have the radius 
\begin{eqnarray}
     r_0 =\frac{\Delta}{k^{1/\rho}}= \frac{(\frac{\epsilon_2}{\alpha})^{\frac{1}{z}} \Delta}{3^{\left \lceil \log s \right \rceil}\cdot 2^{\frac{\left \lceil \log s \right \rceil^2+3\left \lceil \log s \right \rceil+8}{4}}}.
\end{eqnarray}
Together with Lemma~\ref{lemma-L}, the above radius directly implies  $L\le(\frac{\epsilon_2}{\alpha})^{\frac{1}{z}} \Delta$. Based on~\eqref{eq:obj}, we have
\begin{eqnarray}
    \left|\tilde F(\theta)-F(\theta)\right| \le \beta F(\theta) + \epsilon_2 \Delta^z.
\end{eqnarray}
So the set $C_{\nu_0}$ with the weights $\{w_1, \cdots, w_k\}$ yields a $(\beta, \epsilon_2)_z$-coreset.

\section{The Proof of Theorem~\ref{th-main2}}\label{sample-prf}
 Without loss of generality, we assume all the $k$ pseudo-cubes are not empty (otherwise, we can directly remove the empty pseudo-cubes). Then we consider the pseudo-cubes one by one. For $\mathtt{PC}_1$, we directly set $w'_1=w_1=|P\cap \mathtt{PC}_1|$ (in the following analysis, we use $w_i$ and $w'_i$ to denote the exact and approximate weights for $c_{\nu_0, i}$, respectively). Suppose currently we have already obtained the values $w'_1, \cdots, w'_{i_0}$, and try to determine the value for $w'_{i_0+1}$. We define the following notations first. 
 
 \begin{eqnarray}
      I_{i_0}&=&\{i\mid 1\leq i\leq i_0, w'_i>0\};\\
      S_{i_0}&=&\cup_{i\in I_{i_0}}\mathtt{PC}_i;\\
      \tau_{i_0+1}&=&\frac{|P\cap (\mathtt{PC}_{i_0+1}\setminus S_{i_0})|}{|P\cap \mathtt{PC}_{i_0+1}|};\\
      w_{i_0+1}&=&|P\cap (\mathtt{PC}_{i_0+1}\setminus S_{i_0})|.
 \end{eqnarray}
 Obviously, we have $I_1=\{1\}$ and $S_1=\mathtt{PC}_1$. 
 
 Our algorithm for computing the approximate weight $w'_{i_0+1}$ is as follows. 
 We take a uniform sample of $m$ points from $P\cap \mathtt{PC}_{i_0+1}$ by using the sampling technique for relational data~\cite{zhao2018random}. Each sampled point corresponds a binary random variable $x$: if it belongs to $\mathtt{PC}_{i_0+1}\setminus S_{i_0}$, $x=1$; otherwise, $x=0$. Let $g$ be the sum of these $m$ random variables. Suppose $\tau$ is a fixed value $\leq 1/2$ (the exact values of $m$ and $\tau$ will be determined in the following analysis). 
  \begin{itemize}
      \item If $g/m\geq 2\tau$, set $w'_{i_0+1}=g/m\cdot |P\cap \mathtt{PC}_{i_0+1}|$.
      \item Else,  set $w'_{i_0+1}=0$. 
  \end{itemize}
 Informally speaking, if $w'_{i_0+1}>0$, we refer to $\mathtt{PC}_{i_0+1}$ as a ``heavy pseudo-cube''; if $w'_{i_0+1}=0$, we refer to $\mathtt{PC}_{i_0+1}$ as a ``light pseudo-cube''.

 \begin{lemma}
 \label{lem-pt2-1}
 Let $\delta, \lambda\in (0,1)$ and the sample size $m\geq \frac{3}{\delta^2\tau}\log \frac{2}{\lambda}$. Then with probability at least $1-\lambda$, $|w'_{i_0+1}-w_{i_0+1}|$ is no larger than either $\delta\cdot w_{i_0+1}$ or $\frac{2}{1-\delta}\tau\cdot|P\cap \mathtt{PC}_{i_0+1}|$. 
 \end{lemma}
 \begin{proof}
 We consider two cases: (1) $\tau_{i_0+1}\geq \tau$ and (2) $\tau_{i_0+1}< \tau$.
 
 For case (1), since $m\geq \frac{3}{\delta^2\tau}\log \frac{2}{\lambda}$, from the Chernoff bound we know 
 \begin{eqnarray}
      (1-\delta)\tau_{i_0+1}\leq g/m\leq (1+\delta)\tau_{i_0+1} \label{for-pt2-1-1}
 \end{eqnarray}
 with probability at least $1-\lambda$. If the obtained ratio $g/m\geq 2\tau$, according to our algorithm, we have $w'_{i_0+1}=g/m\cdot |P\cap \mathtt{PC}_{i_0+1}|$, i.e., 
 \begin{eqnarray}
     ( 1-\delta)w_{i_0+1}\leq w'_{i_0+1}\leq (1+\delta)w_{i_0+1}. \label{for-pt2-1-2}
 \end{eqnarray}
 If  the obtained ratio $g/m< 2\tau$, according to our algorithm, we have $w'_{i_0+1}=0$. Moreover, from the left-hand side of (\ref{for-pt2-1-1}), we know 
 \begin{eqnarray}
      (1-\delta)\tau_{i_0+1}\leq 2\tau. 
 \end{eqnarray}
 Therefore, $\tau_{i_0+1}\leq\frac{2\tau}{1-\delta}$. That means 
 \begin{eqnarray}
      |w'_{i_0+1}-w_{i_0+1}|=w_{i_0+1}\leq \frac{2}{1-\delta}\tau\cdot|P\cap \mathtt{PC}_{i_0+1}|. \label{for-pt2-1-3}
 \end{eqnarray}
 
 For case (2), from the additive form of the Chernoff bound, we have $g/m\leq 2\tau$ with probability at least $1-\lambda$. Then we have 
 \begin{eqnarray}
      |w'_{i_0+1}-w_{i_0+1}|=w_{i_0+1}\leq  \tau\cdot|P\cap \mathtt{PC}_{i_0+1}|\leq \frac{2}{1-\delta}\tau\cdot|P\cap \mathtt{PC}_{i_0+1}|. \label{for-pt2-1-4}
 \end{eqnarray}
 Combining (\ref{for-pt2-1-2}), (\ref{for-pt2-1-3}), and (\ref{for-pt2-1-4}), we complete the proof. 
 \end{proof}

\begin{lemma}
\label{lem-pt2-2}
Suppose $\frac{2}{1-\delta}\tau<1/k$ and $\tau_{i_0+1}\leq\frac{2\tau}{1-\delta}$. 
There exists at least one $\hat{i}\in I_{i_0}$, such that 
\begin{eqnarray}
     |P\cap \big(\mathtt{PC}_{i_0+1}\cap (S_{\hat{i}}\setminus S_{\hat{i}-1})\big)|\geq\frac{1}{k}|P\cap \mathtt{PC}_{i_0+1}|.
\end{eqnarray}
If $\hat{i}=1$, we set $S_0=\emptyset$. 
\end{lemma}
\begin{proof}
From the definition of the $S_i$s, we have
\begin{eqnarray}
     S_0\subset S_1\subset\cdots\subset S_{i_0}. 
\end{eqnarray}
So we have
\begin{eqnarray}
  \mathtt{PC}_{i_0+1}=(\mathtt{PC}_{i_0+1}\setminus S_{i_0})\bigcup\Big(\cup^{i_0}_{i=1}  \big( \mathtt{PC}_{i_0+1}\cap (S_i\setminus S_{i-1})\big)\Big).  
\end{eqnarray}
It implies 
\begin{eqnarray}
  |P\cap\mathtt{PC}_{i_0+1}|=|P\cap\mathtt{PC}_{i_0+1}\setminus S_{i_0}|+ \sum^{i_0}_{i=1}  |P\cap  \mathtt{PC}_{i_0+1}\cap (S_i\setminus S_{i-1})|.  
\end{eqnarray}
Since $i_0+1\leq k$, from the Pigeonhole principle, we know there exists at least one $\hat{i}\in I_{i_0}$, such that 
$|P\cap \big(\mathtt{PC}_{i_0+1}\cap (S_{\hat{i}}\setminus S_{\hat{i}-1})\big)|\geq\frac{1}{k}|P\cap \mathtt{PC}_{i_0+1}|$.
\end{proof}

Below, we analyze the error induced by the approximate weight $w'_{i_0+1}$. If $g/m\geq 2\tau$, from Lemma~\ref{lem-pt2-1}, we know that $w'_{i_0+1}\in (1\pm \delta)w_{i_0+1}$. So it only induces an extra factor $1\pm \delta$ to the objective value. Thus, we should require $(1+\delta)(1+\beta)\leq 1+\epsilon_1$, i.e., 
\begin{eqnarray}
     \delta\leq \frac{\epsilon_1-\beta}{1+\beta}. \label{for-pt2-5}
\end{eqnarray}
Then we focus on the other case, $w'_{i_0+1}$ is set to be $0$, i.e., $\mathtt{PC}_{i_0+1}$ is a ``light pseudo-cube''. From Lemma~\ref{lem-pt2-2} we know there exists at least one $\hat{i}\in I_{i_0}$, such that 
$|P\cap \big(\mathtt{PC}_{i_0+1}\cap (S_{\hat{i}}\setminus S_{\hat{i}-1})\big)|\geq\frac{1}{k}|P\cap \mathtt{PC}_{i_0+1}|$. Since $\mathtt{PC}_{i_0+1}\cap (S_{\hat{i}}\setminus S_{\hat{i}-1})\subset S_{\hat{i}}\setminus S_{\hat{i}-1}$, we have
\begin{eqnarray}
    |P\cap   (S_{\hat{i}}\setminus S_{\hat{i}-1}) |\geq\frac{1}{k}|P\cap \mathtt{PC}_{i_0+1}|. \label{for-pt2-6}
\end{eqnarray}
Actually, the set $S_{\hat{i}}\setminus S_{\hat{i}-1}=\mathtt{PC}_{\hat{i}}\setminus S_{\hat{i}-1}$, so (\ref{for-pt2-6}) implies
\begin{eqnarray}
   w_{\hat{i}}=|P\cap   (\mathtt{PC}_{\hat{i}}\setminus S_{\hat{i}-1}) |\geq\frac{1}{k}|P\cap \mathtt{PC}_{i_0+1}|. \label{for-pt2-7} 
\end{eqnarray}
Also, since $\mathtt{PC}_{\hat{i}}$ and $\mathtt{PC}_{i_0+1}$ are neighbors, we have
\begin{eqnarray}
     ||c_{\nu_0, \hat{i}}-c_{\nu_0, i_0+1}||\leq 2L.\label{for-pt2-8} 
\end{eqnarray}
As a consequence, for any $\theta$ in the hypothesis space, the error induced by setting $w'_{i_0+1}=0$ is 
\begin{eqnarray}
     w_{i_0+1}\cdot f(\theta, c_{\nu_0, i_0+1})&\leq& \frac{2}{1-\delta}\tau\cdot|P\cap \mathtt{PC}_{i_0+1}|\cdot f(\theta, c_{\nu_0, i_0+1})\nonumber\\
     & \leq &\frac{2}{1-\delta}\tau k\cdot w_{\hat{i}}\big((1+\beta)f(\theta, c_{\nu_0, \hat{i}})+\alpha(2L)^z\big)\nonumber\\
     & = &\boxed{\frac{2\tau k}{1-\delta} (1+\beta)}\cdot w_{\hat{i}}f(\theta, c_{\nu_0, \hat{i}})+\boxed{\frac{2\tau k}{1-\delta} \alpha\cdot (2L)^z}\cdot w_{\hat{i}},\nonumber
\end{eqnarray}
where the first inequality comes from Lemma~\ref{lem-pt2-1}, and the second inequality comes from Assumption~\ref{assump-lip-continue},  (\ref{for-pt2-7}), and (\ref{for-pt2-8}). To guarantee the total multiplicative error no larger than $\epsilon_1$ and the additive error no larger than $\epsilon_2$, we need the following two inequalities for setting the value of $\tau$:
\begin{eqnarray}
      \frac{2\tau k}{1-\delta} (1+\beta)\cdot \mathbf{k}&\leq& \epsilon_1 \\
      \frac{2\tau k}{1-\delta} \alpha\cdot (2L)^z \cdot \mathbf{k}&\leq& \epsilon_2 \Delta^z.
\end{eqnarray}
Note that we add an extra factor $k$ in the above two inequalities, because there are at most $k$ light pseudo-cubes. Based on the fact $(\frac{\Delta}{2L})^z\geq (\frac{1}{2}(\frac{\alpha}{\epsilon_2})^{1/z})^z$ from Theorem~\ref{th-main1} (usually $z$ is a fixed constant), it is sufficient to set 
\begin{eqnarray}
     \tau\leq \Theta(\frac{\epsilon_1}{k^2}).
\end{eqnarray}
Together with (\ref{for-pt2-5}), we obtain the sample size
\begin{eqnarray}
     m\geq \Theta(\frac{k^2}{(\epsilon_1-\beta)^2\epsilon_1}\log\frac{k}{\lambda}),
\end{eqnarray}
where we replace the probability parameter $\lambda$ by $\lambda/k$ to take the union bound over all the $k$ pseudo-cubes.

\section{Assumption 1 for The  Applications}

{\bfseries $k_c$-Clustering.} Let $k_c\in\mathbb{Z}^+$. A feasible solution $\theta$ for the $k_c$-clustering problem is a set of $k_c$ centers in $\mathbb{R}^d$, and each data point is assigned to the nearest center. The objective function of the $k_c$-means clustering problem is  as follows:
\begin{eqnarray}
F(\theta)=\frac{1}{n}\sum^n_{i=1}\min_{c \in \theta}\|p_i-c\|^2_2.
\end{eqnarray}
Similarly, the objective function of the $k_c$-center clustering is
\begin{eqnarray}
F(\theta)=\max_{p\in P}\min_{c \in \theta}|| p - c ||_{2}.
\end{eqnarray}
And the objective function of the $k_c$-median clustering is
\begin{eqnarray}
F(\theta)=\frac{1}{n}\sum^n_{i=1}\min_{c \in \theta}|| p - c ||_{2}.
\end{eqnarray}
Obviously for the $k_c$-center and $k_c$-median problems, we have $\alpha=1, \beta=0$, and $z=1$. For the $k_c$-means problem, we consider any two points $p, q\in \mathbb{R}^d$. Denote by $c_p, c_q\in \theta$  the nearest centers to $p$ and $q$, respectively. Without loss of generality, we can assume $f(\theta, p) \geq f(\theta, q)$. Let $\epsilon\in (0,1)$. Then we have
\begin{eqnarray}
    |f(\theta, p) - f(\theta, q)|&=& \|p-c_p\|_2^2 - \|q-c_q\|_2^2 \nonumber\\
    &\leq& \|p-c_q\|_2^2 - \|q-c_q\|_2^2 \nonumber\\
        &=&\ \|p-q+q-c_q\|_2^2 - \|q-c_q\|_2^2 \nonumber\\
        &=&  \|p-q \|_2^2+2\left\langle p-q, q-c_q\right\rangle\nonumber\\
        &=&  \|p-q \|_2^2+2\left\langle \frac{1}{\sqrt{\epsilon}}(p-q), \sqrt{\epsilon}(q-c_q)\right\rangle\nonumber\\
        &\leq& \|p-q \|_2^2+\frac{1}{ \epsilon}\|p-q \|_2^2+\epsilon\|q-c_q \|_2^2 \nonumber\\
        &=&(1+\frac{1}{ \epsilon})\|p-q \|_2^2+\epsilon f(\theta, q). 
\end{eqnarray}
Therefore, we have $\beta=\epsilon$,  $\alpha=O(\frac{1}{\epsilon})$, and $z=2$ for the $k_c$-means clustering problem.

{\bfseries Logistic Regression.} Logistic regression is a widely used binary classification model  with each data point $p_i$ having the label $y_i\in\{0,1\}$~\cite{cramer2004early}. Denote $g(t):=\frac{1}{1+e^{-t}}$, and
the objective function of logistic regression is: 
\begin{eqnarray}
     F(\theta)=-\frac{1}{n} \sum_{i=1}^{n}\Big(y_{i} \log g\left(\left\langle p_{i}, \theta\right\rangle\right)+\left(1-y_{i}\right) \log \left(1-g\left(\left\langle p_{i}, \theta\right\rangle\right)\right)\Big).
\end{eqnarray}
Note that we  compute the coresets for two classes separately,  i.e., the label can be viewed as a fixed number (either $1$ or $0$). 
Denote $f'(\theta, t)$ as the derivative of $f(\theta, t)$ (where $t = \left\langle p, \theta\right\rangle$). Obviously, $|f'(\theta, t)|<1$. We consider arbitrary two points $p, q\in \mathbb{R}^d$. Also  we assume that they have the same label. Thus, we have
\begin{eqnarray}
     |f(\theta, \left\langle p, \theta\right\rangle) - f(\theta, \left\langle q, \theta\right\rangle)| &\le& |\left\langle p, \theta\right\rangle - \left\langle q, \theta\right\rangle|\nonumber\\
     &=& |\left\langle p-q, \theta\right\rangle|\nonumber\\
     &\le& \|\theta\|_2\cdot\|p-q\|_2.
\end{eqnarray}
Therefore we have $\alpha=O(\|\theta\|_{2}), \beta=0$, and $z=1$.

{\bfseries SVM with Soft Margin.} 
The objective function of the soft margin SVM~\cite{cortes1995support} is as follows:
\begin{eqnarray}
\min_{\boldsymbol{\omega},b,\xi_{i}}& \frac{1}{2}\|\boldsymbol{\omega}\|^2 + \lambda \sum_{i=1}^{n} \xi_{i} \\
s.t.& y_i(\boldsymbol{\omega}^{T}x_i + b) \geq 1 - \xi_{i}\nonumber\\
& \xi_{i} \geq 0, i\in[n].\nonumber
\end{eqnarray}
Specifically, $\lambda$ is a constant number and $\xi_i$ can be set to be hinge loss $\ell_{hinge}(z) = \max(0, 1-z)$ or logistic loss $\ell_{log}(z) = \log(1+\exp(-z))$, where $z = y_i(\boldsymbol{\omega}^{T}x_i + b)$. Through the similar calculations as the logistic regression, we have $\alpha = O(\|\theta\|_{2}), \beta=0, z=1$.

\section{Computational Complexity}
We analyze the total complexity of Algorithm~\ref{alg: aggtree} together with the sampling procedure. It is easy to know that the Gonzalez's algorithm takes $O(kNd)$ time for computing the cluster centers on the $s$ leaf nodes in total, where $N$ is the maximum number of tuples among the input $s$ tables. When merging two nodes $\nu_i$ and $\nu_j$, we first remove the empty grid points from $C_{\nu_i} \times C_{\nu_j}$ by counting the sizes of the pseudo-cubes, where we use $\Psi(N,s,d)$ to denote the complexity of performing the counting one time~\cite{abo2016faq} (for counting an acyclic join, $\Psi(N,s,d) = O\left(s d^{2} N \log (N)\right)$). Also we run the Gonzalez's algorithm on  $\widetilde{C_{\nu_i}\times C_{\nu_j}}$. Note that the distance between any two points of $\widetilde{C_{\nu_i}\times C_{\nu_j}}$ can be obtained in $O(1)$ time, since it can be directly obtained by combining the distances of their projections on $H_{\nu_i}$ and $H_{\nu_j}$ which are respectively stored in $\nu_i$ and $\nu_j$. So the total complexity for merging two nodes is  $O\left(k^2\Psi(N,s,d)+k^3\right)$ (the term ``$k^3$'' is from the complexity of the Gonzalez's algorithm). As a consequence, the complexity for constructing the tree $\mathcal{T}$ is  $O\left(s k^2 \Psi(N,s,d)+sk^3\right)$. Also note that it takes $O(s)$ time for taking a single sample~\cite{zhao2018random}. Each sample is a $d$-dimensional vector and it takes $O(i_0\cdot d)$ time to check that whether it belongs to $\mathtt{PC}_{i_0+1}\setminus(\cup^{i_0}_{i=1}\mathtt{PC}_{i})$ or not. So the sampling procedure takes $\sum^k_{i_0=1}O(ms+i_0\cdot md)=O(k^2md)$ time. Overall, the whole coreset construction takes  $O\left( k^2 (s\Psi(N,s,d) +sk+ m d)\right)$ time. 

\section{Detailed Experimental Results}\label{exp}
% 总述
We evaluate the performance of our relational coreset on three popular machine learning problems, the $k_c$-means clustering, SVM with soft margin, and logistic regression. The experimental results suggest that our method can achieve promising performances on large-scale data sets, where the coreset size is significantly smaller than the design matrix. Moreover, our coreset can be constructed very efficiently with low runtime. In terms of the $k_c$-means clustering problem, comparing with the recently proposed R$k$-means~\cite{DBLP:conf/aistats/CurtinM0NOS20} algorithm, our coreset method can achieve a better solution with comparable construction time and coreset size. All the experimental results were obtained on a server equipped with 3.0GHz Intel CPUs and 384GB main memory. Our algorithms were implemented in Python with PostgreSQL 12.10. We release our codes at Github~\cite{Github}.

% Data sets and queries
\paragraph{Data sets and Queries.} We design four different join queries on the following three real relational data sets.

\textsc{(1)Home Credit~\cite{HomeCredit}} is a relational data set used for credit forecasting. It contains $7$ tables including the historical credit and financial information for each applicant. The data set has the binary labels and we use it to evaluate SVM and logistic regression models. We use $5$ of these tables to design two different queries to extract the design matrix.
\begin{itemize}
    \item \textsc{Query 1 (Q1)} is a multi-way acyclic join that involves $4$ tables. The returned design matrix contains $8.0\times 10^{7}$ rows with $17$ features and the total size is  about $11$GB.
    \item \textsc{Query 2 (Q2)} is a multi-way acyclic join that involves $5$ tables. The returned design matrix contains about $4.0\times 10^{8}$ rows with $19$ features and the total size is about $60$GB.
\end{itemize}
\textsc{(2)Yelp~\cite{Yelp}} is a relational data set that contains the  information of  user reviews in business. The data set has no label so we just use it for the clustering task. We use $3$ main tables to design a join query that forms the design matrix.
\begin{itemize}
    \item \textsc{Query 3 (Q3)} is a chain acyclic foreign key join that involves $3$ tables. The returned design matrix contains about $5.7\times 10^{6}$ rows with $24$ features and the total size is about $1.1$GB.
\end{itemize}
\textsc{(3)Favorita~\cite{Favorita}} is a relational data set that contains the grocery data. We use it for the clustering task. We use $3$ main tables to design a join query that forms the design matrix.
\begin{itemize}
    \item \textsc{Query 4 (Q4)} is a chain acyclic foreign key join that involves $3$ tables. The returned design matrix contains about $4.0\times 10^{7}$ rows with $8$ features and the total size is about $2.4$GB.
\end{itemize}

% Compared method
\paragraph{Baseline methods.} 
We consider five baseline methods for comparison. 
(1) \textsc{Original}: construct the complete design matrix $P$ by preforming the join query, and run the training algorithm directly on $P$;
(2) \textsc{Ori-Gon}: construct the complete design matrix $P$ as \textsc{Original} and then run the training algorithm on the centers obtained by running Gonzalez's algorithm~\cite{DBLP:journals/tcs/Gonzalez85} on $P$;
(3) \textsc{Uniform}: relational uniform sampling algorithm~\cite{zhao2018random};
(4) \textsc{R$k$-means}: the relational $k_c$-means algorithm~\cite{DBLP:conf/aistats/CurtinM0NOS20}. It first performs the $\kappa$-means ($\kappa\in(0,k_c]$) on each table and then constructs a grid coreset of size $\kappa^s$; 
(5) \textsc{RCore}: our proposed relational coreset approach. All experimental results are averaged over 10 trials. 

\paragraph{Results.}
We consider both the running time and optimization quality. We record the end-to-end runtime that includes the design matrix/coreset construction time and the training time. 
For the optimization quality, we take the objective value $F(\theta^{*})$ obtained by  \textsc{Original} as the optimal value of objective function and define  ``$\textbf{Approx.}$''$=\frac{ F(\theta)-F(\theta^{*}) }{F(\theta^{*})}$, where $F(\theta)$ is the objective function value obtained by other baseline methods. We report loss directly when $F(\theta^{*})$ is not available.

\renewcommand{\arraystretch}{1.2}
\setlength\tabcolsep{4pt} 

\begin{table}[!h t b p]
\centering
\begin{tabular}{c|c||c c c c c} 
\toprule[2pt]
\multicolumn{2}{c||}{\textbf{Coreset size}} & $200$ & $400$ & $600$ & $800$ & $1000$\\
\hline
\multirow{4}*{\textbf{\shortstack{End-to-end \\runtime (s)}}} 
& \textsc{Original} & \multicolumn{5}{c}{$>21600$}\\
\cline{2-7}
& \textsc{Ori-Gon} & $3808$ & $5208$ & $6606$ & $8044$ & $9434$\\
\cline{2-7}
& \textsc{Uniform} & $34$ & $35$ & $35$ & $36$ & $38$\\
\cline{2-7}
& \textsc{RCore} & $208$ & $288$ & $363$ & $446$ & $531$\\
\hline
\multirow{3}*{\textbf{Approx.}} 
& \textsc{Ori-Gon} & $1.41$ & $1.50$ & $1.29$ & $1.12$ & $0.95$\\
\cline{2-7}
& \textsc{Uniform} & $2.22$ & $2.60$ & $2.09$ & $2.23$ & $2.14$\\
\cline{2-7}
& \textsc{RCore} & $0.92$ & $0.31$ & $0.27$ & $0.16$ & $0.02$\\
\toprule[2pt]
\end{tabular}
\bigskip
\caption{The results of SVM on $\mathrm{Q1}$.}
\label{apdx-svmQ1}
\end{table}

\begin{table}[!h t b p]
\centering
\begin{tabular}{c|c||c c c c c} 
\toprule[2pt]
\multicolumn{2}{c||}{\textbf{Coreset size}} & $200$ & $400$ & $600$ & $800$ & $1000$\\
\hline
\multirow{4}*{\textbf{\shortstack{End-to-end \\runtime (s)}}} 
& \textsc{Original} & \multicolumn{5}{c}{$3089$}\\
\cline{2-7}
& \textsc{Ori-Gon} & $3808$ & $5208$ & $6606$ & $8044$ & $9434$\\
\cline{2-7}
& \textsc{Uniform} & $34$ & $34$ & $35$ & $36$ & $36$\\
\cline{2-7}
& \textsc{RCore} & $208$ & $288$ & $363$ & $445$ & $530$\\
\hline
\multirow{4}*{\textbf{Approx.}} 
& \textsc{Ori-Gon} & $1.80$ & $2.32$ & $2.01$ & $2.04$ & $1.97$\\
\cline{2-7}
& \textsc{Uniform} & $1.26$ & $1.12$ & $0.88$ & $0.80$ & $0.60$\\
\cline{2-7}
& \textsc{RCore} & $0.43$ & $0.52$ & $0.42$ & $0.24$ & $0.13$\\
\toprule[2pt]
\end{tabular}
\bigskip
\caption{The results of logistic regression on $\mathrm{Q1}$.}
\label{apdx-logQ1}
\end{table}

\begin{table}[!h t b p]
\centering
\begin{tabular}{c|c||c c c c c} 
\toprule[2pt]
\multicolumn{2}{c||}{\textbf{Coreset size}} & $200$ & $400$ & $600$ & $800$ & $1000$\\
\hline
\multirow{4}*{\textbf{\shortstack{End-to-end \\runtime (s)}}} 
& \textsc{Original} & \multicolumn{5}{c}{$>21600$}\\
\cline{2-7}
& \textsc{Ori-Gon} & \multicolumn{5}{c}{$>21600$}\\
\cline{2-7}
& \textsc{Uniform} & $34$ & $35$ & $37$ & $38$ & $39$\\
\cline{2-7}
& \textsc{RCore} & $228$ & $335$ & $430$ & $529$ & $630$\\
\hline
\multirow{4}*{\textbf{Loss}} 
& \textsc{Ori-Gon} & \multicolumn{5}{c}{$-$}\\
\cline{2-7}
& \textsc{Uniform} & $1.09\pm0.21$ & $1.06\pm0.22$ & $1.07\pm0.29$ & $1.04\pm0.20$ & $1.02\pm0.22$\\
\cline{2-7}
& \textsc{RCore} & $0.40\pm0.18$ & $0.33\pm0.14$ & $0.36\pm0.14$ & $0.32\pm0.08$ & $0.24\pm0.03$\\
\toprule[2pt]
\end{tabular}
\bigskip
\caption{The results of SVM on $\mathrm{Q2}$.}
\label{apdx-svmQ2}
\end{table}

\begin{table}[!h t b p]
\centering
\begin{tabular}{c|c||c c c c c} 
\toprule[2pt]
\multicolumn{2}{c||}{\textbf{Coreset size}} & $200$ & $400$ & $600$ & $800$ & $1000$\\
\hline
\multirow{4}*{\textbf{\shortstack{End-to-end \\runtime (s)}}} 
& \textsc{Original} & \multicolumn{5}{c}{$>21600$}\\
\cline{2-7}
& \textsc{Ori-Gon} & \multicolumn{5}{c}{$>21600$}\\
\cline{2-7}
& \textsc{Uniform} & $34$ & $35$ & $36$ & $38$ & $38$\\
\cline{2-7}
& \textsc{RCore} & $228$ & $335$ & $430$ & $529$ & $630$\\
\hline
\multirow{4}*{\textbf{Loss}} 
& \textsc{Ori-Gon} & \multicolumn{5}{c}{$-$}\\
\cline{2-7}
& \textsc{Uniform} & $6.41\pm5.60$ & $5.63\pm3.28$ & $5.70\pm5.00$ & $5.35\pm3.35$ & $5.93\pm2.96$\\
\cline{2-7}
& \textsc{RCore} & $4.13\pm1.70$ & $3.21\pm0.90$ & $3.39\pm1.53$ & $3.17\pm0.56$ & $3.07\pm0.31$\\
\toprule[2pt]
\end{tabular}
\bigskip
\caption{The results of logistic regression on $\mathrm{Q2}$.}
\label{apdx-logQ2}
\end{table}

We consider the SVM and logistic regression (LR) models and compare the four baseline methods except \textsc{R$k$-means} on $\mathrm{Q1}$ and $\mathrm{Q2}$. The results are shown from Table~\ref{apdx-svmQ1} to Table~\ref{apdx-logQ2} respectively.
The results suggest that \textsc{RCore} has better optimization quality in all cases with acceptable running time.

\begin{figure}[ht]
    \centering
    \includegraphics[width=1\linewidth,height=7cm]{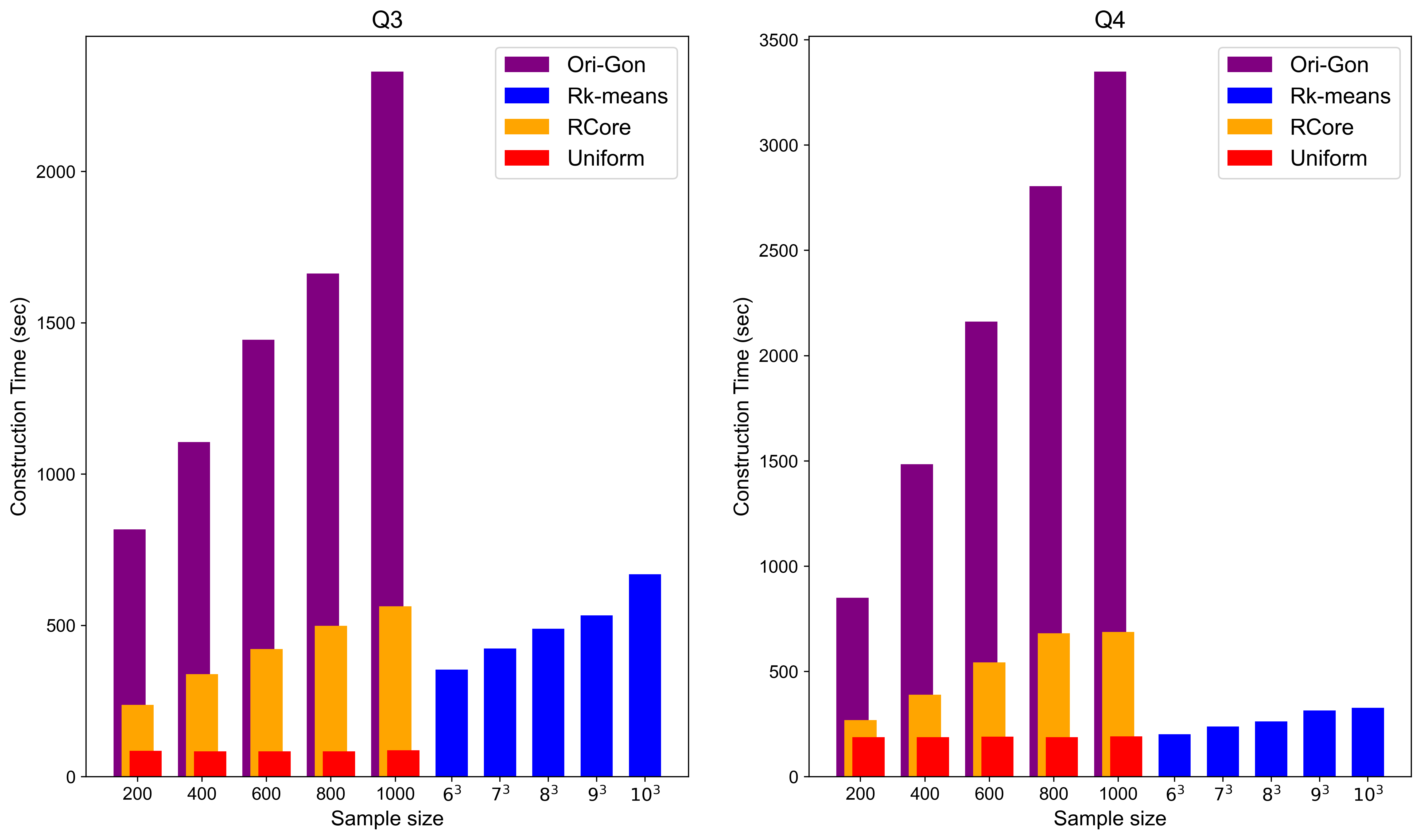}
      \caption{Construction time on $\mathrm{Q3}$ and $\mathrm{Q4}$}
  \label{apdx-fig-time}
  \vspace{-0.1in}
\end{figure}

\begin{figure}[ht]
    \centering
    \includegraphics[width=1\linewidth,height=7cm]{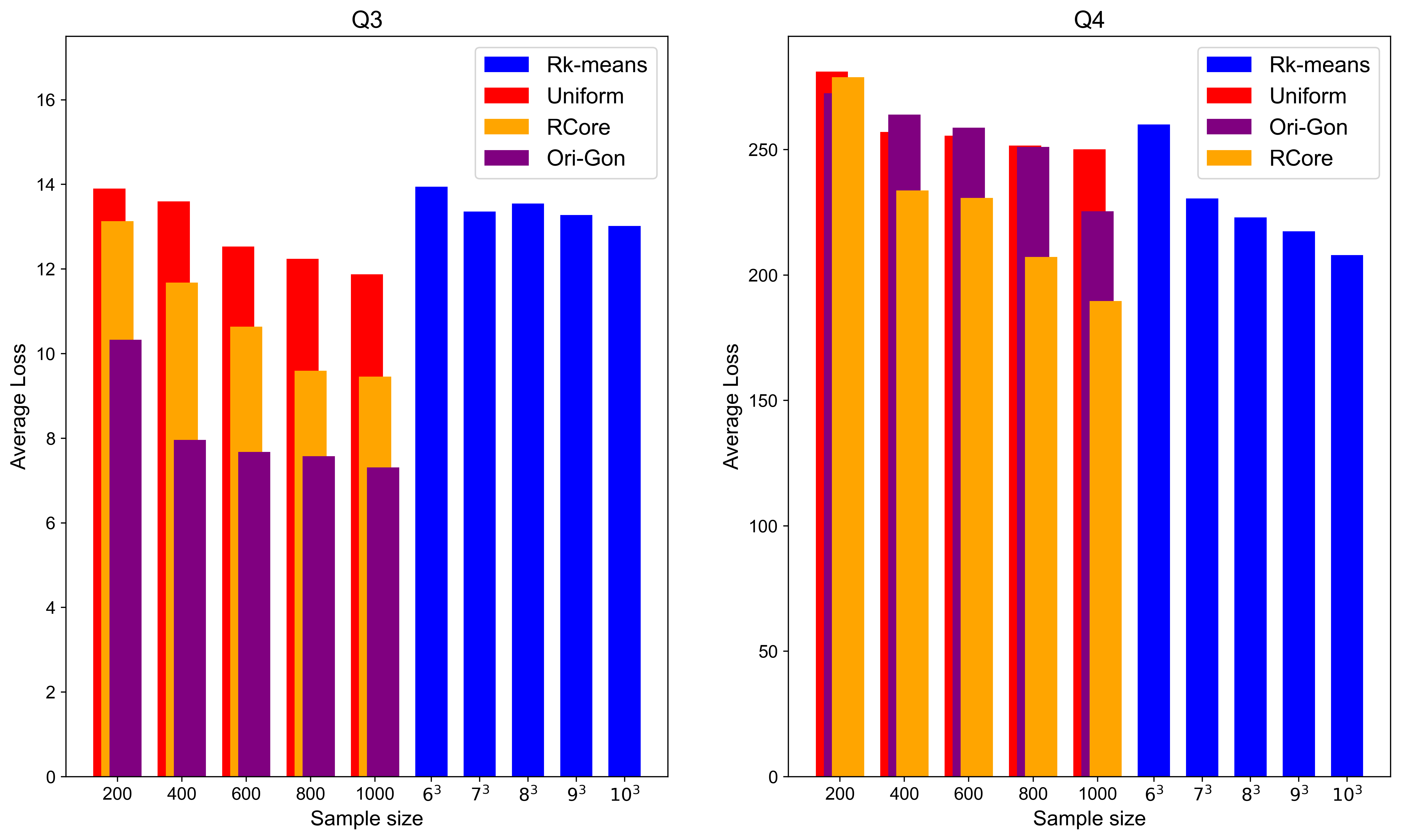}
      \caption{Average loss on $\mathrm{Q3}$ and $\mathrm{Q4}$}
  \label{apdx-fig-loss}
  \vspace{-0.1in}
\end{figure}

For the $k_c$-means clustering problem, we compare the performances of \textsc{RCore}, \textsc{Uniform}, \textsc{Ori-Gon} and \textsc{R$k$-means} on $\mathrm{Q3}$ and $\mathrm{Q4}$.
According to the setting in~\cite{DBLP:conf/aistats/CurtinM0NOS20}, $\kappa$ can be less than $k_c$. In our experiment, we set $k_c=10$, and set $\kappa=\{6,7,8,9,10\}$ for \textsc{R$k$-means}.  Figure~\ref{apdx-fig-time} and~\ref{apdx-fig-loss} illustrate  the obtained coreset construction times and corresponding losses. 
Compared to other baseline methods, \textsc{RCore} has a lower loss in most of the cases with acceptable running time.

\end{document}